\documentclass[11pt]{article}
\usepackage{newtxtext}       
\usepackage{bm}              
\usepackage[a4paper, margin=1in]{geometry}
\usepackage{setspace}
\usepackage{amsmath}
\usepackage{amssymb}
\usepackage{amsfonts}
\usepackage{amsthm}
\usepackage{newtxmath}       
\usepackage{bbm}
\usepackage{algorithm}
\usepackage{algpseudocode}
\usepackage{enumerate}

\usepackage{booktabs}
\usepackage{makecell}
\usepackage{multirow}
\usepackage{array}
\usepackage[table]{xcolor}
\usepackage{siunitx}
\usepackage{graphicx}
\usepackage{subcaption}     
\usepackage{float}
\usepackage{tikz}
\usetikzlibrary{positioning,arrows.meta,fit}

\usepackage{natbib}

\usepackage[
  colorlinks=true,
  linkcolor=black,
  citecolor=black,
  urlcolor=black
]{hyperref}

\usepackage{fancyhdr}
\theoremstyle{plain}
\newtheorem{theorem}{Theorem}[section]
\newtheorem{corollary}[theorem]{Corollary}
\newtheorem{lemma}[theorem]{Lemma}

\theoremstyle{definition}
\newtheorem{assumption}{Assumption}[section]

\newtheorem{remark}[theorem]{Remark}

\numberwithin{equation}{section}

\newcommand{\be}{\begin{eqnarray}}
\newcommand{\ee}{\end{eqnarray}}
\newcommand{\by}{\begin{eqnarray*}}
\newcommand{\ey}{\end{eqnarray*}}
\newcommand{\bn}{\begin{enumerate}}
\newcommand{\en}{\end{enumerate}}
\newcommand{\bi}{\begin{itemize}}
\newcommand{\ei}{\end{itemize}}

\renewcommand{\ge}{\geqslant}
\renewcommand{\le}{\leqslant}
\renewcommand{\geq}{\geqslant}
\renewcommand{\leq}{\leqslant}
\renewcommand{\epsilon}{\varepsilon}
\renewcommand{\cite}{\citet} 

\newcommand{\ignore}[1]{}

\usepackage{booktabs}
\usepackage{tabularx}

\usepackage{enumitem}

\title{\textbf{Generative Distributionally Robust Optimization}}

\author{
    \begin{tabular}{c@{\hspace{2em}}c@{\hspace{2em}}c}
        \large Ziwei Zhang\textsuperscript{1}
        & \large Jonathan Yu-Meng Li\textsuperscript{1}
        & \large Zhihao Jin\textsuperscript{2}\\[0.25em]
        \footnotesize\texttt{zzhan073@uottawa.ca}
        & \footnotesize\texttt{jonathan.li@telfer.uottawa.ca}
        & \footnotesize\texttt{zjin287@uwo.ca}
    \end{tabular}\\[0.9em]
    \normalsize
    \textsuperscript{1}Telfer School of Management, University of Ottawa\\
    \normalsize\textsuperscript{2}Department of Electrical and Computer Engineering,
    Western University
}

\date{July 27, 2026}
\begin{document}

\maketitle
\begin{abstract}
Generative models are increasingly adopted in distributionally robust
optimization (DRO), but existing approaches trade off model
compatibility and adversarial structure: methods that accept arbitrary
samplers do not restrict worst-case laws to a generator family, while
generator-parameterized adversaries rely on model-specific access such
as likelihoods, scores, or training data. We
propose Generative Distributionally Robust Optimization (GDRO), a principled
framework that accepts any sampleable conditional generator as the
nominal model and restricts worst-case laws to a chosen conditional
generator family. The key is the sampler--Sinkhorn pairing: samplers
represent the conditional laws exactly, while Sinkhorn divergence
compares their induced distributions without likelihood access and can
be estimated from samples alone. The resulting population problem
admits a direct finite-sample approximation and differentiable
primal--dual implementation at the active decision context. For
Lipschitz losses, the population Sinkhorn radius bounds downstream
degradation. Across explicit and implicit
generators, our method reduces rare-context inventory regret by 60\%
and SocialGAN navigation collisions by 50\% relative to nominal
decisions.
\end{abstract}

\section{Introduction}
\label{sec:introduction}
\begingroup

Conditional generative models increasingly provide the predictive
distributions used in downstream optimization. Given a current context
$x$, a pretrained model generates demand scenarios, trajectories, or
other uncertain outcomes, and an optimizer chooses a decision from
those samples. This interface is attractive because it accommodates
rich, high-dimensional uncertainty, but it also passes any
misspecification of the generator directly into the decision. The
practical question is therefore not only how to generate realistic
conditional outcomes, but how to robustify a downstream decision when
the available generative model may be wrong.

Distributionally robust optimization (DRO) addresses model
misspecification by optimizing against nearby alternative
distributions~\citep{mohajerin2018wasserstein,rahimian2022frameworks,
kuhn2024distributionally}. Existing approaches, however, face a
tradeoff between \emph{model compatibility} and \emph{adversarial
structure}. A predict-then-robustify approach samples any conditional
generator at the active decision context $x$ and places a Wasserstein
or Sinkhorn ambiguity set around the resulting nominal
law~\citep{mohajerin2018wasserstein,wang2025sinkhorn}. It is
likelihood-free and applies directly at the context where the decision
is made; however, its adversary ranges over an ambient-space transport
ball and is not required to follow the generator's learned structure.
High-dimensional outcomes often concentrate near lower-dimensional
structure---a central premise of modern generative
modeling~\citep{song2019generative}---so worst-case mass may leave the
dependence, dynamics, or manifold structure encoded by the generator.
The transport cost supplies geometry, but it does not impose membership
in the chosen generator family.

Parametric DRO methods~\citep{michel2021modeling,michel2022parametric}
and generative ambiguity-set methods, such as DRO with Generative
Ambiguity Set (GAS-DRO)~\citep{wen2026gasdro} and diffusion
ambiguity-set DRO~\citep{wen2025ddro}, take a complementary approach by
representing the adversarial distribution with a parameterized model.
This preserves a chosen form of parametric or generative structure, but
representative formulations certify the adversary through
model-specific quantities such as likelihood ratios, scores,
reconstruction objectives, or losses evaluated on training data.
Consequently, their stated formulations do not directly provide a
common robustification layer for arbitrary implicit, proprietary, or
otherwise frozen conditional generators available only through
samples. Moreover, a certificate averaged over the training distribution of
contexts---reconstruction loss, for example---need not control the
conditional law at the particular $x$ where a decision is being made.

To address these limitations, we develop a DRO framework that
(i) accepts any sampleable pretrained conditional generator as the
nominal model, (ii) restricts the worst-case law to a chosen conditional
generator family, and (iii) certifies proximity between the two
conditional output laws at the active context. We call the resulting
ambiguity set \emph{generator-faithful}: every admissible worst-case law
is induced by the selected adversarial generator family, rather than
being an arbitrary distribution in the ambient outcome space. When the
nominal architecture is available, a natural choice is to use the same
generator family for both nominal and adversarial models, thereby
preserving its inductive structure. More generally, the framework
accommodates any black-box nominal sampler and permits the adversarial
generator family to be chosen independently.

Our key observation is that the right interface is the
\emph{sampler--Sinkhorn pairing}. Every sampleable conditional
generative model exposes its law operationally through a sampler,
regardless of whether it has a tractable likelihood, score, or density.
Sinkhorn divergence is naturally compatible with this interface: it is
a transport discrepancy between probability laws, can be estimated
from samples alone, and is differentiable with respect to generated
outputs~\citep{cuturi2013sinkhorn,genevay2018sinkhorn,
feydy2019sinkhorn}. It also carries decision-level meaning: under
bounded output support, Sinkhorn proximity controls Wasserstein
proximity, and Kantorovich--Rubinstein duality then bounds the change in
expected loss uniformly over Lipschitz objectives
(Theorem~\ref{thm:decision_alignment}). Sinkhorn divergences have previously been used to
learn generative models from samples; here we use them for a different
purpose---to certify and optimize a worst-case conditional generator
for a downstream decision. The sampler supplies broad model
compatibility, while Sinkhorn supplies a likelihood-free, output-law
transport certificate and a smooth route to adversarial optimization.

\paragraph{Generative distributionally robust optimization (GDRO).}
Let $P_{\hat\phi}(\cdot\mid x)$ denote the conditional law supplied by a
pretrained nominal generator and
$\{Q_\psi(\cdot\mid x):\psi\in\Psi\}$ the conditional laws induced by a
chosen adversarial generator family. Given a downstream loss
$f:\mathcal W\times\mathcal Y\rightarrow\mathbb R$, let
$S_\varepsilon$ denote the debiased Sinkhorn divergence with
regularization $\varepsilon$, and let $\rho$ be the ambiguity radius.
Our central population problem is
\begin{equation}
    w_{\rho,\varepsilon}^{\star}(x)
    \in
    \arg\min_{w\in\mathcal W}\;
    \sup_{\psi\in\Psi:\,
        S_\varepsilon(
            Q_\psi(\cdot\mid x),
            P_{\hat\phi}(\cdot\mid x)
        )\leq\rho}
    \mathbb E_{Y\sim Q_\psi(\cdot\mid x)}
    \!\left[f(w,Y)\right].
    \label{eq:constrained}
\end{equation}
The generator family determines which perturbations are structurally
admissible, Sinkhorn divergence controls their displacement from the
nominal conditional output law, and conditioning throughout on $x$
aligns the stress test with the decision being made. Because both laws
are induced by conditional samplers and Sinkhorn divergence can be
estimated from their output samples, problem~\eqref{eq:constrained}
admits a likelihood-free finite-sample approximation. Section
\ref{sec:method} gives the exact sampler reformulation, its empirical
counterpart, and the resulting optimization algorithm.
\paragraph{Relation to prior work.}
The distinction above positions GDRO between classical
distribution-space DRO, including $\phi$-divergence, Wasserstein, and
Sinkhorn formulations~\citep{mohajerin2018wasserstein,blanchet2019quantifying,
gao2023wasserstein,gao2024regularization,rahimian2022frameworks,
wang2025sinkhorn,yang2025nested}, and structured adversaries based on
parametric likelihood ratios, diffusion ambiguity sets, or
GAS-DRO~\citep{michel2021modeling,michel2022parametric,wen2025ddro,
wen2026gasdro}. FlowDRO~\citep{xu2024flowdro,kobyzev2021normalizing}
uses a normalizing flow to parameterize a distribution-space
Wasserstein adversary, but the Wasserstein ball---rather than membership
in a chosen generator family---still defines which perturbations are
admissible.
Adversarial-environment reinforcement learning~\citep{ren2022dragen}
likewise stress-tests decisions against generated environments, but
perturbs inputs to a fixed simulator. Decision-focused
learning~\citep{donti2017task,elmachtoub2022smart,costa2023portfolio,
wang2025decisionfocused,ma2024drodifferentiable} trains predictive
models for downstream performance, whereas our object is a
post-training ambiguity set around a frozen conditional sampler.
Table~\ref{tab:comparison} isolates the combination that distinguishes
our formulation.

\begin{table}[H]
\centering
\caption{Capabilities native to representative stated formulations.
$\checkmark$: present;\quad $\triangle$: variant-dependent or requiring
model-specific adaptation;\quad $\times$: absent.}
\label{tab:comparison}
\scriptsize
\setlength{\tabcolsep}{3.2pt}
\begin{tabular}{lcccc}
\toprule
Method
  & \makecell{Arbitrary frozen\\nominal sampler}
  & \makecell{Generator-family\\adversary}
  & \makecell{Output-law\\transport}
  & \makecell{Active-context\\certificate} \\
\midrule
\makecell[l]{Conditional Wasserstein /\\Sinkhorn DRO}
  & $\checkmark$ & $\times$ & $\checkmark$ & $\checkmark$ \\
\makecell[l]{Parametric adversary DRO\\
\citep{michel2021modeling,michel2022parametric}}
  & $\times$ & $\triangle$ & $\times$ & $\triangle$ \\
\makecell[l]{Diffusion ambiguity-set DRO / GAS-DRO\\
\citep{wen2025ddro,wen2026gasdro}}
  & $\times$ & $\checkmark$ & $\times$ & $\triangle$ \\
\textbf{GDRO (ours)}
  & $\checkmark$ & $\checkmark$ & $\checkmark$ & $\checkmark$ \\
\bottomrule
\end{tabular}
\end{table}

\paragraph{Contributions.}
\begin{itemize}
  \item We introduce Generative Distributionally Robust Optimization (GDRO),
        a context-local ambiguity framework that combines an arbitrary
        sampleable nominal conditional generator with a
        generator-faithful adversarial family. To our knowledge, this is
        the first context-local DRO framework to restrict worst-case
        laws to a chosen conditional generator family and certify them
        by sample-based Sinkhorn proximity to an arbitrary frozen
        conditional sampler (Section~\ref{sec:method}).
  \item We establish the exact sampler representation of the
        population problem, derive its finite-sample approximation, and
        develop a differentiable primal--dual implementation. For
        Lipschitz losses, the population Sinkhorn radius controls
        downstream degradation, and we analyze convergence of the
        resulting optimization procedure (Section~\ref{sec:theory}).
  \item Across explicit and implicit generators, the resulting robust
        decisions reduce rare-context inventory regret by $60\%$ and
        SocialGAN navigation collisions by $50\%$ relative to nominal
        decisions (Section~\ref{sec:experiments}).
\end{itemize}

\endgroup

\begingroup

\section{Generative Distributionally Robust Optimization (GDRO)}
\label{sec:method}

\subsection{Conditional Population Formulation}
\label{subsec:population_formulation}

Let $\mathcal X$ be the context space, $\mathcal Y$ the outcome space,
$\mathcal W\subseteq\mathbb R^d$ the decision set, and
$f:\mathcal W\times\mathcal Y\to\mathbb R$ the downstream loss. We fix
the active decision context $x\in\mathcal X$ throughout. Write
\[
    P_0^x:=P_{\hat\phi}(\cdot\mid x)
    \qquad\text{and}\qquad
    Q_\psi^x:=Q_\psi(\cdot\mid x),\quad \psi\in\Psi,
\]
for the nominal conditional law and a conditional law in the chosen
adversarial generator family, respectively.

For $\varepsilon>0$, probability measures $P,Q$ on $\mathcal Y$, and
a ground cost $c:\mathcal Y\times\mathcal Y\to\mathbb R_+$,
entropically regularized optimal transport is
\begin{equation}
    \operatorname{OT}_\varepsilon(P,Q)
    :=
    \inf_{\pi\in\Pi(P,Q)}
    \left\{
        \int c(y,y')\,\mathrm d\pi(y,y')
        +\varepsilon\,
        \operatorname{KL}\!\left(\pi\,\middle\|\,P\otimes Q\right)
    \right\},
    \label{eq:population_entropic_ot}
\end{equation}
where $\Pi(P,Q)$ contains all couplings with marginals $P$ and $Q$.
The debiased Sinkhorn divergence~\citep{feydy2019sinkhorn} is
\begin{equation}
    S_\varepsilon(P,Q)
    :=
    \operatorname{OT}_\varepsilon(P,Q)
    -\tfrac12\operatorname{OT}_\varepsilon(P,P)
    -\tfrac12\operatorname{OT}_\varepsilon(Q,Q).
    \label{eq:population_sinkhorn}
\end{equation}
By construction, $S_\varepsilon(P,P)=0$. For boundedly supported laws
on a Euclidean outcome space and the squared-Euclidean cost
$c(y,y')=\|y-y'\|_2^2$, $S_\varepsilon(P,Q)\geq0$ and
$S_\varepsilon(P,Q)\to W_2^2(P,Q)$ as
$\varepsilon\downarrow0$~\citep{feydy2019sinkhorn,
peyre2019computational}.

For an ambiguity radius $\rho\geq0$, the conditional population problem
is
\begin{equation}
    V_{\rho,\varepsilon}(x)
    :=
    \inf_{w\in\mathcal W}
    \sup_{\substack{\psi\in\Psi:\\
        S_\varepsilon(Q_\psi^x,P_0^x)\leq\rho}}
    \mathbb E_{Y\sim Q_\psi^x}[f(w,Y)].
    \label{eq:population_gdro}
\end{equation}
The feasible laws in~\eqref{eq:population_gdro} are
\emph{generator-faithful}: every one is induced by the selected
conditional generator family. The natural choice, when the nominal
architecture is available, is to use the same conditional architecture
for the adversary and include the nominal parameter in $\Psi$; then the
ambiguity set is nonempty for every $\rho\geq0$. A different adversarial
family is also permitted, provided it contains at least one conditional
law within radius $\rho$ of $P_0^x$.

\subsection{Sampler Representation and Empirical Approximation}
\label{subsec:finite_sample}

Let the nominal and adversarial conditional samplers be
\[
    G_{\hat\phi}:\mathcal Z_0\times\mathcal X\to\mathcal Y,
    \quad Z_0\sim\zeta_0,
    \qquad
    G_\psi:\mathcal Z_{\mathrm A}\times\mathcal X\to\mathcal Y,
    \quad Z_{\mathrm A}\sim\zeta_{\mathrm A}.
\]
Their conditional output laws are the pushforwards
\begin{equation}
    P_0^x=\bigl(G_{\hat\phi}(\cdot,x)\bigr)_\#\zeta_0,
    \qquad
    Q_\psi^x=\bigl(G_\psi(\cdot,x)\bigr)_\#\zeta_{\mathrm A}.
    \label{eq:conditional_pushforwards}
\end{equation}
In the natural same-architecture case, the samplers may share the same
latent space and base law. The notation in
\eqref{eq:conditional_pushforwards} also covers a black-box nominal
sampler and a separately chosen adversarial family.

Substituting~\eqref{eq:conditional_pushforwards} into
\eqref{eq:population_gdro} gives the exact sampler representation
\begin{equation}
    V_{\rho,\varepsilon}(x)
    =
    \inf_{w\in\mathcal W}
    \sup_{\substack{\psi\in\Psi:\\
        S_\varepsilon(
            (G_\psi(\cdot,x))_\#\zeta_{\mathrm A},
            (G_{\hat\phi}(\cdot,x))_\#\zeta_0
        )\leq\rho}}
    \mathbb E_{Z_{\mathrm A}\sim\zeta_{\mathrm A}}
    \!\left[f\!\left(w,G_\psi(Z_{\mathrm A},x)\right)\right].
    \label{eq:sampler_population}
\end{equation}
Because the Sinkhorn constraint compares the induced output laws, this
representation requires no likelihoods, scores, or pointwise pairing
of generated outputs.

To obtain the finite-sample problem, draw
\[
    Y_{0,i}\stackrel{\mathrm{iid}}{\sim}P_0^x,
    \qquad
    Z_{{\mathrm A},i}\stackrel{\mathrm{iid}}{\sim}\zeta_{\mathrm A},
    \quad i=1,\dots,M.
\]
They define the empirical conditional output laws
\begin{equation}
    \widehat P_{G_{\hat\phi},x}
    :=
    \frac1M\sum_{i=1}^M\delta_{Y_{0,i}},
    \qquad
    \widehat P_{G_\psi,x}
    :=
    \frac1M\sum_{i=1}^M
    \delta_{G_\psi(Z_{{\mathrm A},i},x)}.
    \label{eq:empirical_conditional_laws}
\end{equation}
The nominal and adversarial batches may be drawn independently. The
adversarial latent draws are held fixed across $\psi$ during the
sample-average optimization.

For empirical measures
$\alpha=M^{-1}\sum_i\delta_{a_i}$ and
$\beta=M^{-1}\sum_j\delta_{b_j}$, let
\[
    \Pi_M
    :=
    \left\{
        \pi\in\mathbb R_+^{M\times M}:
        \pi\mathbf 1_M=M^{-1}\mathbf 1_M,\;
        \pi^\top\mathbf 1_M=M^{-1}\mathbf 1_M
    \right\}.
\]
The empirical regularized transport term is
\begin{equation}
    \operatorname{OT}_\varepsilon(\alpha,\beta)
    =
    \min_{\pi\in\Pi_M}
    \left\{
        \sum_{i,j=1}^M\pi_{ij}c(a_i,b_j)
        +
        \varepsilon\sum_{i,j=1}^M
        \pi_{ij}\log(M^2\pi_{ij})
    \right\},
    \label{eq:discrete_entropic_ot}
\end{equation}
with the convention $0\log0=0$. The two empirical self-transport terms
in~\eqref{eq:population_sinkhorn} are computed analogously. Since
$\sum_{i,j}\pi_{ij}=1$, the regularizer in
\eqref{eq:discrete_entropic_ot} equals
$\varepsilon\sum_{i,j}\pi_{ij}\log\pi_{ij}
+2\varepsilon\log M$. The additive constant appears in each of the
three transport terms in~\eqref{eq:population_sinkhorn} and therefore
cancels from $S_\varepsilon$. Hence the empirical Sinkhorn divergence
is unchanged under the common entropy convention
$\varepsilon\sum_{i,j}\pi_{ij}\log\pi_{ij}$.

The empirical objective is
\begin{equation}
    \widehat F_M(w,\psi)
    :=
    \frac1M\sum_{i=1}^M
    f\!\left(w,G_\psi(Z_{{\mathrm A},i},x)\right),
    \label{eq:saa_obj}
\end{equation}
and the finite-sample GDRO problem is
\begin{equation}
    \widehat V_M(x)
    :=
    \inf_{w\in\mathcal W}
    \sup_{\substack{\psi\in\Psi:\\
        S_\varepsilon(
            \widehat P_{G_\psi,x},
            \widehat P_{G_{\hat\phi},x}
        )\leq\rho}}
    \widehat F_M(w,\psi).
    \label{eq:constrained_method}
\end{equation}

\subsection{Primal--Dual Optimization}
\label{subsec:primal_dual}

For fixed $w$, set
\[
    \widehat s(\psi)
    :=
    S_\varepsilon(
        \widehat P_{G_\psi,x},
        \widehat P_{G_{\hat\phi},x}).
\]
The empirical inner problem is
\begin{equation}
    \sup_{\psi\in\Psi}\widehat F_M(w,\psi)
    \quad\text{subject to}\quad
    \widehat s(\psi)\leq\rho.
    \label{eq:inner}
\end{equation}
For a nonlinear generator this problem is generally nonconcave and its
feasible set has no closed-form projection. Its Lagrangian is
\begin{equation}
    \widehat{\mathcal L}(w,\psi,\mu)
    :=
    \widehat F_M(w,\psi)
    -\mu\bigl(\widehat s(\psi)-\rho\bigr),
    \qquad \mu\geq0.
    \label{eq:lagrangian}
\end{equation}
The associated dual problem minimizes over $\mu\geq0$ after maximizing
over $\psi$. Accordingly, we ascend in $\psi$ and use projected dual
descent,
\[
    \psi\leftarrow
    \psi+\eta_\psi\nabla_\psi\widehat{\mathcal L},
    \qquad
    \mu\leftarrow
    \left[
        \mu+\eta_\mu
        \bigl(\widehat s(\psi)-\rho\bigr)
    \right]_+.
\]
The multiplier increases when $\widehat s(\psi)>\rho$ and decreases
otherwise. In computing $\nabla_\psi\widehat s(\psi)$, the nominal
samples are fixed; only the adversarial samples
$G_\psi(Z_{{\mathrm A},i},x)$ depend on $\psi$. For
$\varepsilon>0$ and a smooth ground cost, the empirical Sinkhorn
divergence is differentiable in these generated samples, so the chain
rule propagates the gradient through $G_\psi$.

\begin{algorithm}[H]
\captionsetup{labelsep=colon}
\caption{Primal--Dual GDRO}
\label{alg:GDRO}
\begin{algorithmic}[1]
\Require Context $x$; nominal sampling oracle
         $Y\sim P_0(\cdot\mid x)$; adversarial sampler $G_\psi$;
         radius $\rho$; batch size $M$; step sizes
         $\eta_\psi,\eta_\mu,\eta_w$; iteration counts $J,K,T$.
\State Draw the nominal and adversarial batches in
       \eqref{eq:empirical_conditional_laws}; form
       $\widehat P_{G_{\hat\phi},x}$.
\State Initialize $w\in\mathcal W$, $\psi\in\Psi$, and choose
       $\mu^{(1)}>0$.
\For{$t=1,\dots,T$}
    \State $w^{(t)}\leftarrow w$; \quad $\mu\leftarrow\mu^{(1)}$.
    \For{$k=1,\dots,K$}
        \For{$j=1,\dots,J$}
            \State Form $\widehat P_{G_\psi,x}$ from the fixed
                   adversarial latent draws.
            \State $\psi\leftarrow
                   \psi+\eta_\psi\nabla_\psi
                   \widehat{\mathcal L}(w,\psi,\mu)$.
        \EndFor
        \State Set $\psi_k\leftarrow\psi$ and evaluate
               $\widehat s(\psi_k)$.
        \State $\mu\leftarrow
               [\mu+\eta_\mu(
               \widehat s(\psi_k)-\rho)]_+$.
        \State $\widehat g_k\leftarrow
               \nabla_w\widehat F_M(w,\psi_k)$.
    \EndFor
    \State $w\leftarrow
           \Pi_{\mathcal W}\!\left(
           w-\eta_w K^{-1}\sum_{k=1}^K\widehat g_k\right)$.
\EndFor
\State \Return $\overline w=T^{-1}\sum_{t=1}^T w^{(t)}$.
\end{algorithmic}
\end{algorithm}

\paragraph{Model access and scope.}
Algorithm~\ref{alg:GDRO} queries the nominal conditional model only for
output samples, which may be cached; it requires no likelihood, score,
or gradient access to that model and therefore accommodates implicit or
non-differentiable nominal generators. The gradient implementation does
require an optimizable adversarial family, typically one whose outputs
are differentiable with respect to $\psi$. Because the empirical inner
problem is nonconcave, the algorithm generally targets an approximate
stationary solution rather than a certified global maximizer.

\endgroup

\section{Theoretical Analysis}
\label{sec:theory}

We establish four guarantees for Algorithm~\ref{alg:GDRO}, all at a
fixed context $x$: a bound on the Sinkhorn ambiguity
(Lemma~\ref{lem:sinkhorn_bound}); a decision-alignment result linking
the radius $\rho$ to downstream loss
(Theorem~\ref{thm:decision_alignment}); convergence of the inner
adversarial maximization (Theorem~\ref{thm:inner_convergence}); and
joint convergence of the alternating procedure under outer convexity
(Theorem~\ref{thm:joint_convergence}). Proofs are deferred to
Appendix~\ref{app:proofs}.

\subsection{Setup}
\label{sec:setup}

\paragraph{Standing assumptions.}
Throughout, the context $x$ is fixed and:
\begin{itemize}[leftmargin=2.2em,itemsep=1pt,topsep=2pt]
\item[(A1)] $f(w,y)$ is $L_f$-Lipschitz in $y$, uniformly in $w$;
\item[(A2)] $\mathcal{W}\subset\mathbb{R}^d$ is convex and compact with
diameter $D_{\mathcal{W}}$;
\item[(A3)] {nominal and adversarial generator outputs are
uniformly bounded: $\|G_{\hat\phi}(z_0,x)\|_2\le R$ for
$\zeta_0$-almost every $z_0$, and $\|G_\psi(z_{\mathrm A},x)\|_2\le R$
for every $\psi\in\Psi$ and $\zeta_{\mathrm A}$-almost every
$z_{\mathrm A}$.}
\end{itemize}
All Sinkhorn and Wasserstein quantities in this section use the
squared-Euclidean ground cost $c(y,y')=\|y-y'\|_2^2$. Convexity of
$f$ in $w$ is assumed only for joint convergence
(Assumption~\ref{ass:outer_convex}). We take $\varepsilon>0$ and
$\rho>0$ throughout.

\paragraph{Notation.}
At the population level, retain the notation $P_0^x$ and $Q_\psi^x$
from Section~\ref{subsec:population_formulation}. For the optimization
analysis, condition on the batches in
\eqref{eq:empirical_conditional_laws} and abbreviate
\[
    F(w,\psi):=\widehat F_M(w,\psi),\qquad
    F_0(w):=\frac1M\sum_{i=1}^M f(w,Y_{0,i}),
\]
\[
    s(\psi):=S_\varepsilon(\widehat P_{G_\psi,x},
    \widehat P_{G_{\hat\phi},x}),\qquad
    \mathcal A:=\{\psi\in\Psi:s(\psi)\le\rho\}.
\]
Let $\phi(w):=\max_{\psi\in\mathcal A}F(w,\psi)$,
$\phi^*:=\min_{w\in\mathcal W}\phi(w)$, and
$\psi^*(w)\in\arg\max_{\psi\in\mathcal A}F(w,\psi)$. We assume these
optima are attained and $\mathcal A$ is nonempty. Finally, write
$s_k:=s(\psi_k)$, ${\gamma_M}:=\varepsilon\log M$,
$B:=\max\{\rho,4R^2\}$, and let $\mu^{(1)}>0$ be the multiplier at the
start of each inner block.

\paragraph{Sinkhorn--Wasserstein closeness.}
{
For the two equally weighted $M$-point empirical laws used by the
algorithm, the comparison is explicit:
\begin{equation}
    \bigl|S_\varepsilon(\alpha,\beta)-W_2^2(\alpha,\beta)\bigr|
    \le {\gamma_M}=\varepsilon\log M.
    \label{eq:empirical_sw_gap}
\end{equation}
Indeed, an optimal unregularized permutation coupling has relative
entropy $\log M$, and the same bound applies to both self-transport
terms. A proof is given in Appendix~\ref{app:proof_empirical_sw_gap}.
Unlike a generic asymptotic interpolation-gap statement,
\eqref{eq:empirical_sw_gap} gives a uniform finite-sample comparison
over generated empirical clouds, with an explicit error term
$\gamma_M=\varepsilon\log M$.
}
\subsection{Sinkhorn Bound and Decision Alignment}
\label{subsec:decision_alignment}

We first bound the Sinkhorn divergence between any adversarial conditional law and the nominal conditional
law. This bound is the one problem-specific ingredient
of the convergence analysis; everything downstream of it is standard.

\begin{lemma}[Uniform Sinkhorn upper bound]
\label{lem:sinkhorn_bound}
{Under~\textup{(A3)}, any two nominal or adversarial output
laws $P,Q$ satisfy}
\[
    {0\le S_\varepsilon(P,Q)\le 4R^2.}
\]
{In particular, $0\le s(\psi)\le4R^2$ for every
$\psi\in\Psi$.}
\end{lemma}

\begin{proof}[Proof sketch]
{
The product coupling $P\otimes Q$ has zero relative-entropy penalty
and transport cost at most $4R^2$. Thus
$\operatorname{OT}_\varepsilon(P,Q)\le4R^2$; subtracting the two
nonnegative self-transport terms cannot increase the result.
Nonnegativity is the standard positivity property of the debiased
Sinkhorn divergence. This argument does not require a common latent
space or paired draws. Full proof in
Appendix~\ref{app:proof_sinkhorn_bound}.
}
\end{proof}

The next result formalizes the central claim of the introduction: the
radius $\rho$ is a budget on downstream performance.

{
Let $\mathcal P_R$ be the probability laws supported on the closed
Euclidean ball of radius $R$, and define
\begin{equation}
    \omega_{\varepsilon,R}(r)
    :=\sup\{W_1(P,Q):P,Q\in\mathcal P_R,
    \ S_\varepsilon(P,Q)\le r\}.
    \label{eq:sinkhorn_modulus}
\end{equation}

\begin{theorem}[Population decision alignment]
\label{thm:decision_alignment}
The function $\omega_{\varepsilon,R}$ is finite and nondecreasing, and
$\omega_{\varepsilon,R}(r)\downarrow0$ as $r\downarrow0$. Under
\textup{(A1)--(A3)}, every $\psi\in\Psi$ satisfying
$S_\varepsilon(Q_\psi^x,P_0^x)\le\rho$ obeys
\[
    \left|\mathbb E_{Q_\psi^x}[f(w,Y)]
    -\mathbb E_{P_0^x}[f(w,Y)]\right|
    \le L_f\,\omega_{\varepsilon,R}(\rho)
    \qquad\text{for every }w\in\mathcal W.
\]
\end{theorem}

\begin{proof}[Proof sketch]
On the compact space $\mathcal P_R$, Sinkhorn divergence is continuous,
positive definite, and metrizes weak convergence; $W_1$ is also
continuous. Compactness therefore implies
$\omega_{\varepsilon,R}(r)\downarrow0$. The loss bound then follows
from Kantorovich--Rubinstein duality. Full proof in
Appendix~\ref{app:proof_decision_alignment}.
\end{proof}

\begin{corollary}[Empirical decision alignment]
\label{cor:empirical_decision_alignment}
Under~\textup{(A1)--(A3)}, every $\psi\in\mathcal A$ satisfies
\[
    |F(w,\psi)-F_0(w)|
    \le L_f\sqrt{\rho+{\gamma_M}}
    \qquad\text{for every }w\in\mathcal W.
\]
\end{corollary}
}

\begin{remark}[Significance]
The radius $\rho$ is not an abstract divergence budget but a direct,
loss-uniform bound on the downstream loss any adversary in
$\mathcal{A}$ can extract, holding simultaneously for every Lipschitz
objective. This makes $\rho$ an interpretable design parameter.
\end{remark}

\subsection{Inner Convergence}
\label{subsec:inner_convergence}

For fixed $w$, the inner problem is a nonconcave maximization over a
nonconvex feasible set, solved by the primal--dual scheme of
Algorithm~\ref{alg:GDRO}: at each of $K$ dual steps, $J$ ascent steps
approximately solve $\max_\psi\mathcal{L}(w,\psi,\mu_k)$, followed by a
projected update of $\mu_k$. The next result bounds the suboptimality in
objective value, averaged over the dual iterates.

{
Because gradient ascent need not find a global maximizer of a
nonconcave Lagrangian, the guarantee is stated under an explicit inner
oracle condition:
\begin{equation}
    \psi_k\in\arg\max_{\psi\in\Psi}
    \widehat{\mathcal L}(w,\psi,\mu_k).
    \label{eq:exact_inner_oracle}
\end{equation}
An approximate oracle adds its average Lagrangian error to the bound
below; in particular, an average error of order $K^{-1/2}$ preserves
the stated rate.
}

\begin{theorem}[Inner convergence]
\label{thm:inner_convergence}
{Under~\textup{(A1)--(A3)} and
\eqref{eq:exact_inner_oracle}, use projected dual descent with
$\eta_\mu=\mu^{(1)}/(B\sqrt K)$. Then}
\[
    F\bigl(w,\psi^*(w)\bigr)
    - \frac{1}{K}\sum_{k=1}^{K} F(w,\psi_k)
    \;\le\;
    {\frac{B\,\mu^{(1)}}{\sqrt{K}},}
\]
where $\psi^*(w)\in\arg\max_{\psi\in\mathcal{A}}F(w,\psi)$.
\end{theorem}

The proof is the standard dual-subgradient analysis of constrained
saddle-point problems~\citep{zinkevich2003online,nedic2009subgradient}:
a {projected-descent} regret bound on the dual variable,
into which the
Sinkhorn bound (Lemma~\ref{lem:sinkhorn_bound}) enters only through the
constant $4R^2$. See Appendix~\ref{app:proof_inner}.

\subsection{Joint Convergence}
\label{subsec:joint_convergence}

The full algorithm alternates the inner primal--dual updates with an
outer projected step along the averaged subgradient
$\hat g^{(t)}=\frac1K\sum_{k=1}^{K}\hat g_k$, $\hat g_k\in\partial_w
F(w^{(t)},\psi_k)$. Coupling the two layers requires convexity of the
downstream loss in $w$.

\begin{assumption}[Outer convexity]
\label{ass:outer_convex}
$f(w,y)$ is convex in $w$ for every $y$, and {the outer
subgradients used by the algorithm are assumed uniformly bounded,}
$\|\hat g_k\|\le G_w$.
\end{assumption}

{
For fixed $w$, let
\[
    Q_w(\mu):=\sup_{\psi\in\Psi}
    \{F(w,\psi)+\mu(\rho-s(\psi))\}.
\]
\begin{assumption}[Dual regularity]
\label{ass:dual_regularity}
For every $w\in\mathcal W$, $Q_w$ has a selected minimizer
$\mu^*(w)\ge0$ satisfying
$\overline\mu^*:=\sup_{w\in\mathcal W}\mu^*(w)<\infty$.
\end{assumption}
}

The inner iterates need not be feasible; the next lemma controls their
mean constraint violation using the standard long-term-constraint
argument~\citep{mahdavi2012trading}, with an explicit comparison
constant $C_K$ derived as in GAS-DRO~\citep{wen2026gasdro}.

\begin{lemma}[Average constraint slack]
\label{lem:slack}
Under the conditions of Theorem~\ref{thm:inner_convergence}
{and Assumption~\ref{ass:dual_regularity}},
\[
    \frac{1}{K}\sum_{k=1}^{K} s_k \;\le\; \rho + \frac{C_K}{\sqrt{K}},
    \qquad {
    C_K = \frac{\max\{\rho,4R^2\}}{\mu_C-\overline\mu^*}
    \Bigl(\tfrac{\mu^{(1)}}{2}+\tfrac{|\mu_C-\mu^{(1)}|^2}{2\mu^{(1)}}\Bigr),
    }
\]
{where $\mu_C>\overline\mu^*$ is fixed.}
\end{lemma}

\begin{theorem}[Joint convergence]
\label{thm:joint_convergence}
Under~\textup{(A1)--(A3)}, Assumptions~\ref{ass:outer_convex}
and~\ref{ass:dual_regularity}, and the inner-oracle condition
\eqref{eq:exact_inner_oracle}, run Algorithm~\ref{alg:GDRO} for
$T$ outer iterations with $K$ inner steps and outer step size
$\eta_w=D_{\mathcal{W}}/(G_w\sqrt{T})$. Then the averaged iterate
$\bar w=T^{-1}\sum_{t=1}^{T}w^{(t)}$ satisfies
\begin{align*}
    \phi(\bar w)-\phi^*
    \;\le\;
    &\underbrace{\frac{D_{\mathcal W}G_w}{\sqrt T}}_{\textnormal{outer}}
    +
    \underbrace{\frac{B\mu^{(1)}}{\sqrt K}}_{\textnormal{inner}}
    +
    \underbrace{\frac{L_f C_K}{2\sqrt{(\rho+\gamma_M)K}}}_{\textnormal{slack}}
    \\[2pt]
    &+
    \underbrace{2L_f\sqrt{\rho+\gamma_M}}_{\textnormal{comparison residual}} .
\end{align*}
\end{theorem}

\begin{remark}[Interpretation]
The first three terms are \emph{optimization errors}: outer projected
subgradient descent, inner Lagrangian approximation, and accumulated
average slack. These terms vanish as $T,K\to\infty$; for instance,
taking $K=T^2$ makes the inner and slack terms $O(1/T)$, leaving the
dominant optimization rate $O(1/\sqrt{T})$. The last term,
$2L_f\sqrt{\rho+\gamma_M}$, is a \emph{comparison residual} in the
present proof. It has the same scale as the empirical
decision-alignment bound in
Corollary~\ref{cor:empirical_decision_alignment} and arises from
bounding the loss gap through Kantorovich--Rubinstein duality and the
nominal empirical law. This term is independent of $T$ and $K$ in our
analysis; we do not claim that it is intrinsic to the algorithm.
\end{remark}

%

\section{Experiments}
\label{sec:experiments}

\paragraph{Overview.}
We evaluate GDRO on two downstream decision problems.  The first is
contextual newsvendor, where the decision is convex and the optimal
order for a fixed scenario distribution has a weighted-quantile form.
We study this task on a synthetic benchmark with controlled rare
conditional tails and on the real M5 retail demand dataset
\citep{makridakis2022m5}.  The second downstream problem is robot
navigation among pedestrians, where the decision problem is nonconvex
and the nominal model is the official implicit SocialGAN trajectory
generator \citep{gupta2018social}.  These two settings test GDRO under
both explicit conditional generators and implicit black-box samplers.

\paragraph{Baselines.}
For all tasks, \textsc{Nominal} optimizes the downstream decision using
samples from the fitted nominal generator.  \textsc{KL} performs
finite-support KL reweighting over nominal samples, and therefore
cannot create new outcomes outside the sampled support.  \textsc{W2}
uses a finite-sample Wasserstein-2 perturbation of the sampled outcomes.
For the newsvendor experiments, where the nominal model has an explicit
conditional decoder, we also compare with \textsc{GAS-DRO}.  For the
synthetic benchmark only, we report an \textsc{Oracle} that optimizes
using samples from the true data-generating process.  Our method,
\textsc{GDRO}, optimizes against adversarial generators constrained by a
context-local Sinkhorn certificate around the nominal generator.  Full
implementation details and hyperparameters are given in
Appendix~\ref{app:experiments}.

\paragraph{Metrics.}
For newsvendor, the primary decision-quality metrics are average
realized cost and, when an oracle is available, regret relative to the
oracle. Since the asymmetric newsvendor loss already encodes the
underage--overage trade-off, lower tail losses or stockout rates alone
do not necessarily indicate a better policy: they can also be obtained
by systematically over-ordering. We therefore report Q95 loss,
CVaR$_{95}$, stockout rate, and average order quantity as diagnostic
metrics that reveal the risk--conservatism trade-off. For robot
navigation, we report collision count, near-miss count, arrival count,
evaluation loss, minimum realized pedestrian distance, and terminal
goal error.

\subsection{Contextual Newsvendor}

\paragraph{Synthetic rare-tail benchmark.}
The synthetic benchmark is designed so that most contexts follow a
regular demand pattern, while rare contexts may activate a shifted tail
component.  This stresses a common failure mode of nominal generative
models: the fitted generator can match the bulk of the conditional
distribution while underrepresenting rare high-cost outcomes.  The
downstream decision is a multi-product order vector, and the realized
loss is the asymmetric newsvendor cost.

Table~\ref{tab:synthetic_newsvendor_main} reports results separately on
frequent and rare contexts.  On frequent contexts, GDRO obtains the
lowest average cost and regret among non-oracle methods while keeping
the average order close to the oracle.  On rare contexts, GDRO again
achieves the lowest non-oracle regret, improving substantially over the
nominal policy, KL reweighting, W2, and GAS-DRO. KL achieves lower tail losses and stockout rates in some slices, but it
does so with larger average orders and substantially higher average
cost. This illustrates the risk--conservatism trade-off in newsvendor:
tail-risk reductions are meaningful only when considered together with
realized cost and order inflation. This supports the
central mechanism of GDRO: rather than only reweighting nominal samples
or freely perturbing support points, it searches over nearby
generator-induced conditional laws that expose consequential tail
scenarios.

\begin{table}[t]
\centering
\caption{Synthetic contextual newsvendor results. Avg. cost and regret
are the primary decision-quality metrics. Q95 loss, CVaR$_{95}$,
stockout, and average order are diagnostic metrics for the
risk--conservatism trade-off.}
\label{tab:synthetic_newsvendor_main}
\begin{tabular}{llcccccc}
\toprule
\textbf{Slice} & \textbf{Method}
& \textbf{Avg. cost}
& \textbf{Regret}
& \textbf{Q95 loss}
& \textbf{CVaR$_{95}$}
& \textbf{Stockout}
& \textbf{Avg. order} \\
\midrule
\multirow{6}{*}{Frequent}
 & Oracle        & 422.99 & 0.00  & 801.89 & 1260.96 & 0.084 & 25.09 \\
 & Nominal       & 436.55 & 13.55 & 947.01 & 1429.04 & 0.104 & 24.35 \\
 & KL            & 451.43 & 28.44 & 754.84 & 1136.56 & 0.073 & 26.34 \\
 & W2            & 436.73 & 13.74 & 846.07 & 1292.49 & 0.088 & 25.08 \\
 & GAS-DRO       & 431.21 & 8.22  & 889.10 & 1358.03 & 0.094 & 24.68 \\
 & \textbf{GDRO} & \textbf{429.62} & \textbf{6.63} & 809.49 & 1257.77 & 0.084 & 25.16 \\
\midrule
\multirow{6}{*}{Rare}
 & Oracle        & 514.46 & 0.00  & 967.81  & 1799.44 & 0.084 & 28.35 \\
 & Nominal       & 537.09 & 22.63 & 1291.28 & 2259.68 & 0.124 & 26.50 \\
 & KL            & 541.09 & 26.63 & 1074.20 & 1903.99 & 0.101 & 28.25 \\
 & W2            & 532.88 & 18.41 & 1205.79 & 2128.83 & 0.117 & 27.03 \\
 & GAS-DRO       & 530.36 & 15.89 & 1236.14 & 2190.00 & 0.118 & 26.75 \\
 & \textbf{GDRO} & \textbf{523.47} & \textbf{9.01} & 1148.88 & 2064.01 & 0.107 & 27.26 \\
\bottomrule
\end{tabular}
\end{table}

\paragraph{M5 retail demand.}
We next apply the same newsvendor pipeline to the M5 retail demand
dataset. We select 20 products with the fewest zero-demand observations
to make conditional scenario generation stable. For each day $t$, the
context is a rolling window of the previous 10 days of demand, and the
nominal CVAE--LSTM generator produces conditional scenarios for the
next-day demand $Y_t$. The downstream decision is the order vector for
that day, and all methods are evaluated out of sample on held-out
rolling test windows. Unlike the synthetic benchmark, the true
data-generating distribution is unknown, so performance is measured
only against realized held-out demand. The robustness radius and
inner-loop budget are adapted using a historical average-cost risk
score, as described in Appendix~\ref{app:m5}.

GDRO achieves the lowest average realized cost on the M5 test
set. It also improves the moderate-tail metrics Q$_{90}$ and
CVaR$_{90}$ relative to the nominal policy while keeping the average
order close to W2. KL and GAS-DRO reduce some extreme-tail or stockout
metrics, but mainly by placing larger orders; in particular, GAS-DRO
attains the lowest stockout rate with a much larger average order and
the highest average cost. These results indicate that GDRO
improves the primary decision objective without relying on excessive
over-ordering.

\begin{table*}[t]
\centering
\caption{Performance comparison on the held-out M5 contextual
newsvendor test set. Average realized cost is the primary
decision-quality metric; tail losses, stockout, and average order
diagnose the risk--conservatism trade-off.}
\label{tab:m5_newsvendor_main}
\resizebox{\textwidth}{!}{
\begin{tabular}{lrrrrrrr}
\toprule
Method
& Avg. cost
& Q$_{90}$
& Q$_{95}$
& CVaR$_{90}$
& CVaR$_{95}$
& Stockout
& Avg. order \\
\midrule
Nominal
& 1230.30
& 1729.04
& 2634.58
& 2772.97
& 3440.73
& 0.1088
& 64.64 \\

KL
& 1459.00
& 1951.36
& 2326.20
& 2607.80
& 3101.56
& 0.1070
& 71.33 \\

W2
& 1212.91
& 1654.29
& 2400.24
& 2541.50
& 3093.66
& 0.0940
& 66.39 \\

GAS-DRO
& 1712.35
& 2510.35
& 2592.51
& 2705.35
& 2859.42
& 0.0193
& 86.44 \\

\textbf{GDRO}
& \textbf{1203.02}
& 1575.37
& 2354.47
& 2454.90
& 3035.12
& 0.0840
& 67.07 \\
\bottomrule
\end{tabular}
}
\end{table*}

\subsection{Pedestrian-Conditioned Robot Navigation with SocialGAN}

\paragraph{Setup.}
To test GDRO with an implicit high-dimensional generator, we use the
official pretrained SocialGAN model \citep{gupta2018social}.  Given
8 observed frames of pedestrian motion, SocialGAN generates samples of
the next 12 frames.  We add a downstream robot navigation decision on
top of this prediction task: the robot chooses a 12-step velocity plan
to move from a sampled start point to a goal while avoiding pedestrians. During optimization, the robot only observes the past trajectory and SocialGAN-generated futures; the true future is used only for
evaluation. In evaluation, a collision is counted when the robot comes within $0.5$ meters of any realized pedestrian trajectory, a near miss is counted using a larger $0.8$-meter buffer, and arrival means that the
robot reaches within $0.25$ meters of the goal within the planning horizon.

We construct 200 fixed decision tasks from 50 SocialGAN test contexts,
with four start--goal queries per context.  Queries are selected using
only nominal-sample information so that the straight-line robot path has
nontrivial predicted collision risk, while the start and goal are not
initially occupied.  This creates difficult but feasible navigation
instances rather than trivially safe cases.  Details of query
construction, loss weights, adaptive radii, and optimization budgets are
given in Appendix~\ref{app:socialgan}.

\paragraph{Results.}
Figure~\ref{fig:sgan_mechanism} shows the qualitative mechanism on a
real SocialGAN test scene. Nominal and KL both collide with the
realized pedestrian future: Nominal trusts the baseline forecast, while
KL can only reweight the same nominal support. GDRO avoids the
pedestrian while still reaching the goal. In contrast, W2 perturbs
support points directly in trajectory space, which can move adversarial
futures off the SocialGAN manifold and break the learned
trajectory-interaction structure. The resulting robot plan is more
distorted and not necessarily safer under realized evaluation.

\begin{figure}[t]
\centering
\includegraphics[width=\linewidth]{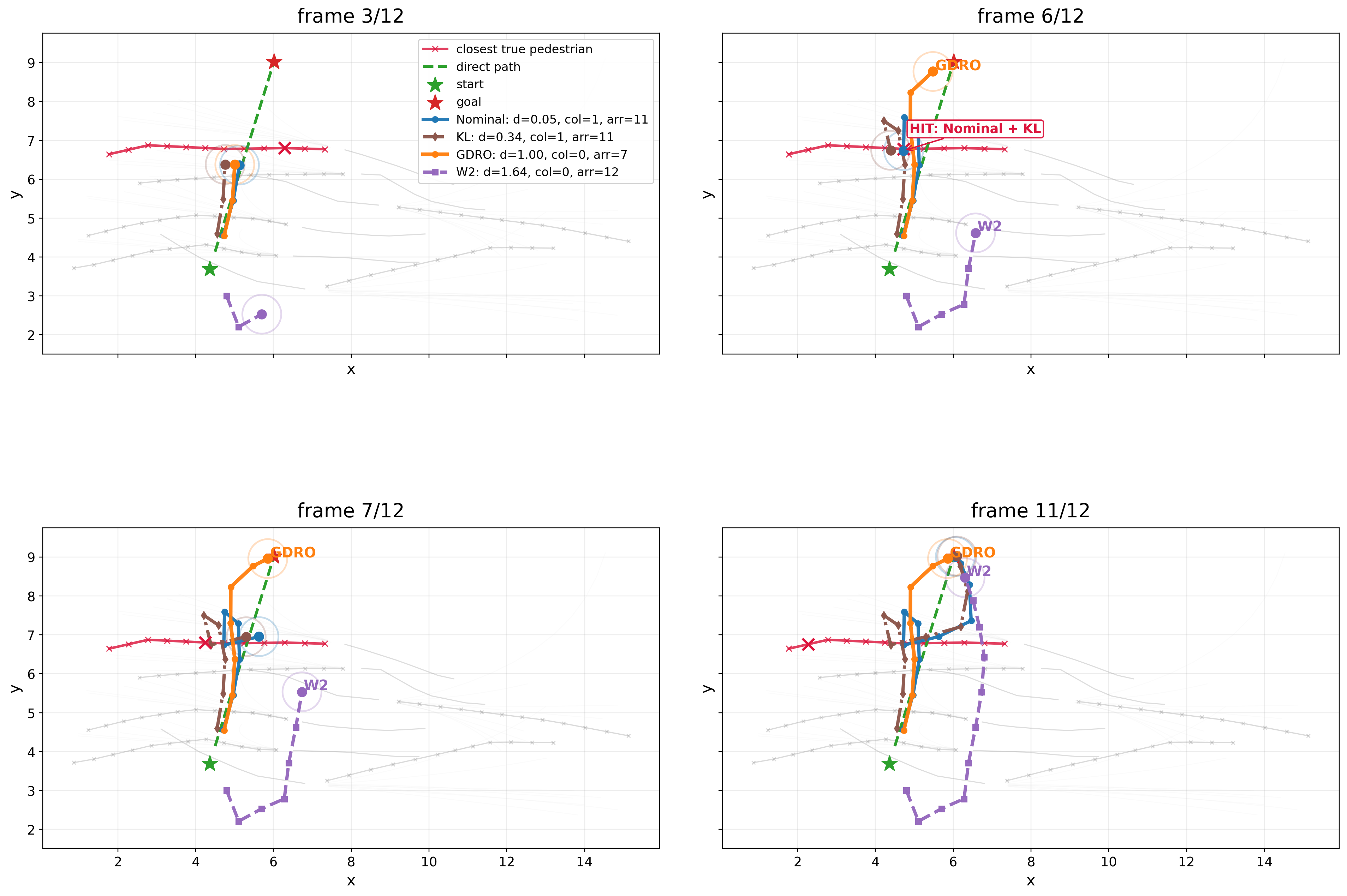}
\caption{
Mechanism example on a real SocialGAN test scene. Nominal and KL plans
collide with the realized pedestrian future. GDRO avoids the pedestrian
while still reaching the goal, whereas W2 produces a distorted route
after unconstrained output-space perturbation.
}
\label{fig:sgan_mechanism}
\end{figure}

Table~\ref{tab:sgan_robot_main} confirms this pattern across 200 fixed
test tasks. A collision is counted at $0.5$m, a near miss at $0.8$m,
and arrival means reaching within $0.25$m of the goal. GDRO reduces
collisions from $22/200$ to $11/200$ and near misses from $106/200$ to
$55/200$, while still arriving in $195/200$ tasks. W2 achieves a larger
average minimum distance, but has more collisions, a much lower arrival
rate, and a larger terminal gap, showing that conservatism in the wrong
geometry can be harmful.

\begin{table}[t]
\centering
\caption{SocialGAN robot navigation on 200 fixed test tasks.
Collisions are counted at $0.5$m, near misses at $0.8$m, and arrivals
at a $0.25$m goal threshold.}
\label{tab:sgan_robot_main}
\begin{tabular}{lrrrrrr}
\toprule
Method & Collisions & Near misses & Arrivals & Eval. loss & Min dist. & Terminal gap \\
\midrule
Nominal & 22/200 & 106/200 & 200/200 & 35.48 & 0.800 & 0.010 \\
KL      & 14/200 & 90/200  & 200/200 & 29.84 & 0.851 & 0.016 \\
\textbf{GDRO} & \textbf{11/200} & \textbf{55/200} & 195/200 & \textbf{26.24} & 0.925 & 0.035 \\
W2      & 34/200 & 72/200  & 142/200 & 56.60 & 1.017 & 0.277 \\
\bottomrule
\end{tabular}
\end{table}

\paragraph{Takeaway.}
Across newsvendor and SocialGAN navigation, GDRO improves downstream
robustness by stress-testing decisions through structured
generator-level perturbations. The results show that robustness is not
only about being more conservative, but about being conservative in the
right geometry: GDRO preserves the simulator's learned structure,
whereas free output-space perturbations can produce misleading stress
tests.


\section{Conclusion}
\label{sec:conclusion}

We proposed GDRO, a simulation-based distributionally robust
optimization framework for decision-making with learned conditional
generators. The core idea is to define robustness directly on generated
samples at the active decision context. Instead of relying on
likelihood ratios, reconstruction certificates, or global
model-level discrepancies, GDRO compares nominal and adversarial output
clouds through a sample-based Sinkhorn certificate. This makes the
ambiguity set context-local, transport-based, and applicable even when
the nominal model is only available as a sampler.

Theoretically, we showed that the Sinkhorn radius admits a downstream
loss interpretation for Lipschitz objectives and proved a finite-time
guarantee for the fixed empirical objective when the outer decision
layer is convex, as in the newsvendor setting. Empirically, GDRO
improves the primary decision objective in contextual newsvendor tasks
and reduces realized collision risk in SocialGAN navigation, showing
that the same sample-based framework can also be applied beyond convex
decision problems to implicit pretrained generators. Together, these
results highlight the central mechanism of GDRO: robustness is not
merely about being more conservative, but about stress-testing the
simulator in the right geometry. By constraining adversarial scenarios
through generator-induced output laws, GDRO preserves sample-level
structure that unrestricted support perturbations may destroy.

\bibliographystyle{plainnat}
\bibliography{sinkhorn_sampler_dro_refs}

\appendix


\section{Proofs}
\label{app:proofs}

This appendix contains full proofs for the results stated 
in Section~\ref{sec:theory}. For convenience, each theorem 
or lemma is restated before its proof.

{
The population proof below concerns $P_0^x$ and $Q_\psi^x$ at the
fixed context $x$. All optimization proofs condition on the fixed
Monte Carlo batches and use the empirical shorthand introduced in
Section~\ref{sec:setup}.
}

\paragraph{Basic assumptions.}
We use the following basic assumptions throughout the theoretical
analysis:
\begin{enumerate}
    \item[(A1)] $f(w,y)$ is $L_f$-Lipschitz in $y$ uniformly in
    $w \in \mathcal{W}$.
    
    \item[(A2)] $\mathcal{W} \subset \mathbb{R}^d$ is convex and
    compact with diameter $D_{\mathcal{W}}$.
    
    \item[(A3)] {Nominal and adversarial generator outputs
    are uniformly bounded as stated in Section~\ref{sec:setup}.}
\end{enumerate}

\paragraph{Notation.}
For a fixed context $x$, we use the following notation:
\begin{itemize}[leftmargin=2em,itemsep=2pt]
    \item {$F(w,\psi):=\widehat F_M(w,\psi)$ denotes the
    empirical downstream loss under adversary $\psi$; this shorthand
    is used only in the fixed-batch optimization proofs.}

    \item $\mathcal{A}(\hat\phi,\rho,x)
    :=
    \left\{
    \psi :
    S_\varepsilon
    \left(
        \widehat{P}_{G_\psi,x},
        \widehat{P}_{G_{\hat\phi},x}
    \right)
    \le \rho
    \right\}$
    denotes the Sinkhorn ambiguity set. When $\hat\phi$, $\rho$,
    and $x$ are clear from context, we write simply $\mathcal{A}$.

    \item $\phi(w) := \max_{\psi \in \mathcal{A}} F(w,\psi)$
    denotes the robust outer objective, and
    $\phi^* := \min_{w \in \mathcal{W}} \phi(w)$ denotes its
    optimal value.

    \item $\psi^*(w) \in \arg\max_{\psi \in \mathcal{A}} F(w,\psi)$
    denotes a worst-case adversary for decision $w$, and
    $w^* \in \arg\min_{w \in \mathcal{W}} \phi(w)$ denotes an
    optimal robust decision.

    \item $s_k :=
    S_\varepsilon
    \left(
        \widehat{P}_{G_{\psi_k},x},
        \widehat{P}_{G_{\hat\phi},x}
    \right)$
    denotes the Sinkhorn divergence at the $k$-th inner iterate.

    \item {${\gamma_M}:=\varepsilon\log M$ and
    $B:=\max\{\rho,4R^2\}$.}

    \item When analyzing the outer update,
    $\hat g_k \in \partial_w F(w^{(t)},\psi_k)$ denotes the
    outer subgradient associated with the $k$-th inner iterate, and
    \[
        \hat g^{(t)}
        :=
        \frac{1}{K}\sum_{k=1}^K \hat g_k
    \]
    denotes the averaged outer subgradient used in
    Algorithm~\ref{alg:GDRO}.
\end{itemize}

Additional assumptions required only for the joint convergence theorem
are stated explicitly in the proof of
Theorem~\ref{thm:joint_convergence}.

\subsection{Proof of Lemma~\ref{lem:sinkhorn_bound}}
\label{app:proof_sinkhorn_bound}

{
\begin{lemma}[Uniform Sinkhorn upper bound; restatement of
Lemma~\ref{lem:sinkhorn_bound}]
\label{lem:sinkhorn_bound_restated}
Under Assumption~\textup{(A3)}, any two nominal or adversarial output
laws $P,Q$ satisfy $0\le S_\varepsilon(P,Q)\le4R^2$.
\end{lemma}

\begin{proof}
The independent coupling $P\otimes Q$ is feasible in
\eqref{eq:population_entropic_ot} and has zero KL penalty. Since both
laws are supported in the radius-$R$ ball,
\[
    \|y-y'\|_2^2\le(\|y\|_2+\|y'\|_2)^2\le4R^2.
\]
Consequently,
$\operatorname{OT}_\varepsilon(P,Q)\le4R^2$. The two self-transport
terms are nonnegative, so
$S_\varepsilon(P,Q)\le\operatorname{OT}_\varepsilon(P,Q)\le4R^2$.
Nonnegativity follows from the positivity of the debiased Sinkhorn
divergence for squared-Euclidean cost~\citep{feydy2019sinkhorn}.
\end{proof}

\subsection{Proof of the Empirical Sinkhorn--Wasserstein Comparison}
\label{app:proof_empirical_sw_gap}

\begin{proof}
Let $\alpha=M^{-1}\sum_i\delta_{a_i}$ and
$\beta=M^{-1}\sum_j\delta_{b_j}$. The KL term in
\eqref{eq:discrete_entropic_ot} is nonnegative, hence
$\operatorname{OT}_\varepsilon(\alpha,\beta)\ge
W_2^2(\alpha,\beta)$. The unregularized problem has an optimal
permutation coupling. Its $M$ nonzero entries equal $1/M$, so its KL
divergence from the product weights $1/M^2$ is $\log M$. Therefore
\[
    0\le
    \operatorname{OT}_\varepsilon(\alpha,\beta)
    -W_2^2(\alpha,\beta)
    \le\varepsilon\log M.
\]
Applying the same argument to the two self-transport terms shows that
$S_\varepsilon(\alpha,\beta)-W_2^2(\alpha,\beta)$ is one number in
$[0,{\gamma_M}]$ minus half of each of two numbers in
$[0,{\gamma_M}]$. It therefore lies in
$[-{\gamma_M},{\gamma_M}]$, proving~\eqref{eq:empirical_sw_gap}.
\end{proof}
}

\subsection{Proof of Theorem~\ref{thm:decision_alignment}}
\label{app:proof_decision_alignment}

{
\begin{proof}
The radius-$R$ ball is compact, so $\mathcal P_R$ is compact under weak
convergence. On this space, $S_\varepsilon$ is continuous, positive
definite, and metrizes weak convergence~\citep[Theorem~1]{feydy2019sinkhorn};
$W_1$ is also continuous.

If $\omega_{\varepsilon,R}(r)$ did not tend to zero, there would be
$r_n\downarrow0$ and $P_n,Q_n\in\mathcal P_R$ with
$S_\varepsilon(P_n,Q_n)\le r_n$ but $W_1(P_n,Q_n)$ bounded away from
zero. Compactness gives a subsequence converging to some $P,Q$.
Continuity gives $S_\varepsilon(P,Q)=0$, hence $P=Q$, while continuity
of $W_1$ gives $W_1(P_n,Q_n)\to0$, a contradiction. Thus
$\omega_{\varepsilon,R}(r)\downarrow0$.

For a population-feasible $\psi$, definition
\eqref{eq:sinkhorn_modulus} gives
$W_1(Q_\psi^x,P_0^x)\le\omega_{\varepsilon,R}(\rho)$.

Assumption~\textup{(A1)} and Kantorovich--Rubinstein duality now yield
\[
    \left|\mathbb E_{Q_\psi^x}[f(w,Y)]
    -\mathbb E_{P_0^x}[f(w,Y)]\right|
    \le L_f W_1(Q_\psi^x,P_0^x)
    \le L_f\omega_{\varepsilon,R}(\rho).
\]
\end{proof}

\begin{proof}[Proof of Corollary~\ref{cor:empirical_decision_alignment}]
For $\psi\in\mathcal A$, Assumption~\textup{(A1)},
Kantorovich--Rubinstein duality, $W_1\le W_2$, and
\eqref{eq:empirical_sw_gap} give
\[
    |F(w,\psi)-F_0(w)|
    \le L_fW_1(\widehat P_{G_\psi,x},\widehat P_{G_{\hat\phi},x})
    \le L_f\sqrt{s(\psi)+{\gamma_M}}
    \le L_f\sqrt{\rho+{\gamma_M}}.
\]
\end{proof}
}

\subsection{Proof of Theorem~\ref{thm:inner_convergence}}
\label{app:proof_inner}

\begin{theorem}[Inner convergence; restated]
{Under~\textup{(A1)--(A3)} and
\eqref{eq:exact_inner_oracle}, run the inner loop for $K$ dual steps
with $\eta_\mu=\mu^{(1)}/(B\sqrt K)$. Then}
\[
    F\bigl(w, \psi^*(w)\bigr)
    - \frac{1}{K}\sum_{k=1}^{K} F(w, \psi_k)
    \;\le\;
    \frac{{B}\,\mu^{(1)}}{\sqrt{K}},
\]
where $\psi^*(w) \in \arg\max_{\psi\in\mathcal{A}} F(w,\psi)$ is the
inner optimum and $\mu^{(1)} > 0$ is the initial dual variable.
\end{theorem}

\begin{proof}
The inner loop is projected primal--dual ascent on the Lagrangian
$\mathcal{L}(w,\psi,\mu) = F(w,\psi) - \mu\,(s(\psi)-\rho)$: at each
step $k$, $\psi_k$ maximizes $\mathcal{L}(w,\cdot,\mu_k)$ and
{the dual variable is updated by projected descent,
$\mu_{k+1}=[\mu_k-\eta_\mu b_k]_+$, with
$b_k:=\rho-s_k$.}

We use the standard duality-gap analysis for saddle-point subgradient
problems~\citep{zinkevich2003online,nedic2009subgradient}.

\textbf{Dual trajectory.}
{
Nonexpansiveness of projection gives, for any reference $\mu\ge0$,
\begin{equation}
    \frac{1}{K}\sum_{k=1}^K (\mu_k - \mu)\, b_k
    \;\le\;
    \frac{(\mu^{(1)}-\mu)^2}{2K\eta_\mu}
    +\frac{\eta_\mu}{2K}\sum_{k=1}^K |b_k|^2.
    \label{eq:projected_dual_regret}
\end{equation}
By Lemma~\ref{lem:sinkhorn_bound}, $0\le s_k\le4R^2$, so
$|b_k|\le B$. Taking $\mu=0$ and
$\eta_\mu = \mu^{(1)}/(B\sqrt{K})$ yields
\begin{equation}
    \frac{1}{K}\sum_{k=1}^K \mu_k\, b_k
    \;\le\; \frac{B\,\mu^{(1)}}{\sqrt{K}}.
    \label{eq:dual_bound}
\end{equation}
}

\textbf{Duality gap.}
For every $\mu_k \ge 0$, weak duality gives
\[
    Q(\mu_k) := \max_{\psi}\bigl[F(w,\psi) + \mu_k\,(\rho - s(\psi))\bigr]
    \;\ge\; \max_{\psi\in\mathcal{A}} F(w,\psi) = F\bigl(w,\psi^*(w)\bigr),
\]
since any feasible $\psi\in\mathcal{A}$ satisfies $s(\psi)\le\rho$ and
hence $\mu_k(\rho-s(\psi))\ge 0$. As $\psi_k$ maximizes the Lagrangian
at $\mu_k$,
\[
    F(w,\psi_k) + \mu_k b_k \;=\; Q(\mu_k) \;\ge\; F\bigl(w,\psi^*(w)\bigr).
\]
Averaging over $k=1,\dots,K$ and rearranging,
\[
    F\bigl(w,\psi^*(w)\bigr) - \frac{1}{K}\sum_{k=1}^K F(w,\psi_k)
    \;\le\; \frac{1}{K}\sum_{k=1}^K \mu_k b_k
    \;\le\; \frac{B\,\mu^{(1)}}{\sqrt{K}},
\]
where the last step is~\eqref{eq:dual_bound}.
{An approximate oracle contributes its average
Lagrangian error additively, as noted after
\eqref{eq:exact_inner_oracle}.}
\end{proof}

\subsection{Proof of Lemma~\ref{lem:slack}}
\label{app:proof_slack}

The argument below follows the standard long-term-constraint analysis
of \citet{mahdavi2012trading}. The comparison-multiplier construction is
likewise standard and is used, for example, in the GAS-DRO analysis of
\citet{wen2026gasdro}. We include the relevant details here to keep the
Sinkhorn specialization self-contained.

\begin{proof}
Fix the outer decision $w$ and omit it from the notation. Define
\[
    b_k:=\rho-s_k,
    \qquad
    v_k:=s_k-\rho=-b_k .
\]
By Lemma~\ref{lem:sinkhorn_bound}, $0\le s_k\le4R^2$, and hence
\[
    |b_k|\le B:=\max\{\rho,4R^2\}.
\]

We first upper-bound the projected-dual regret term. Applying
\eqref{eq:projected_dual_regret} with reference multiplier
$\mu=\mu_C$ gives
\[
    \frac1K\sum_{k=1}^K(\mu_k-\mu_C)b_k
    \le
    \frac{|\mu_C-\mu^{(1)}|^2}{2K\eta_\mu}
    +
    \frac{\eta_\mu}{2K}\sum_{k=1}^K |b_k|^2 .
\]
Using $|b_k|\le B$ and
$\eta_\mu=\mu^{(1)}/(B\sqrt K)$ yields
\begin{equation}
    \frac1K\sum_{k=1}^K(\mu_k-\mu_C)b_k
    \le
    \frac{B}{\sqrt K}
    \left(
        \frac{\mu^{(1)}}{2}
        +
        \frac{|\mu_C-\mu^{(1)}|^2}{2\mu^{(1)}}
    \right).
    \label{eq:slack_upper}
\end{equation}

We now lower-bound the same term by the average constraint violation.
Recall the dual function
\[
    Q_w(\mu)
    :=
    \sup_{\psi\in\Psi}
    \{F(w,\psi)+\mu(\rho-s(\psi))\}.
\]
By the exact inner oracle condition~\eqref{eq:exact_inner_oracle},
\begin{equation}
    Q_w(\mu_k)=F(w,\psi_k)+\mu_k b_k .
    \label{eq:Q_oracle}
\end{equation}
Let $\mu^*(w)$ be the selected minimizer of $Q_w$ from
Assumption~\ref{ass:dual_regularity}. Since
$Q_w(\mu_k)\ge Q_w(\mu^*(w))$, \eqref{eq:Q_oracle} implies
\begin{equation}
    \mu_k b_k
    =
    Q_w(\mu_k)-F(w,\psi_k)
    \ge
    Q_w(\mu^*(w))-F(w,\psi_k).
    \label{eq:mukbk_lower}
\end{equation}
Moreover, evaluating the supremum defining $Q_w(\mu^*(w))$ at
the candidate $\psi_k$ gives
\begin{equation}
    F(w,\psi_k)+\mu^*(w)b_k
    \le
    Q_w(\mu^*(w)).
    \label{eq:Q_candidate}
\end{equation}
Since $b_k=-v_k$, \eqref{eq:Q_candidate} is equivalent to
\begin{equation}
    F(w,\psi_k)
    \le
    Q_w(\mu^*(w))+\mu^*(w)v_k .
    \label{eq:F_upper}
\end{equation}
Substituting \eqref{eq:F_upper} into \eqref{eq:mukbk_lower} gives
\[
    \mu_k b_k
    \ge
    Q_w(\mu^*(w))
    -
    \bigl[Q_w(\mu^*(w))+\mu^*(w)v_k\bigr]
    =
    -\mu^*(w)v_k .
\]
Therefore, using $b_k=-v_k$,
\begin{align}
    \sum_{k=1}^K(\mu_k-\mu_C)b_k
    &=
    \sum_{k=1}^K\mu_k b_k
    -
    \mu_C\sum_{k=1}^K b_k \notag\\
    &=
    \sum_{k=1}^K\mu_k b_k
    +
    \mu_C\sum_{k=1}^K v_k \notag\\
    &\ge
    -\mu^*(w)\sum_{k=1}^K v_k
    +
    \mu_C\sum_{k=1}^K v_k \notag\\
    &=
    \bigl(\mu_C-\mu^*(w)\bigr)\sum_{k=1}^K v_k .
    \label{eq:slack_lower}
\end{align}

Combining \eqref{eq:slack_lower} with \eqref{eq:slack_upper}, and
dividing by $\mu_C-\mu^*(w)>0$, gives
\[
    \frac1K\sum_{k=1}^K v_k
    \le
    \frac{B}{(\mu_C-\mu^*(w))\sqrt K}
    \left(
        \frac{\mu^{(1)}}{2}
        +
        \frac{|\mu_C-\mu^{(1)}|^2}{2\mu^{(1)}}
    \right).
\]
By Assumption~\ref{ass:dual_regularity},
$\mu^*(w)\le\overline\mu^*$, and $\mu_C>\overline\mu^*$. Hence
\[
    \frac1K\sum_{k=1}^K v_k
    \le
    \frac{B}{(\mu_C-\overline\mu^*)\sqrt K}
    \left(
        \frac{\mu^{(1)}}{2}
        +
        \frac{|\mu_C-\mu^{(1)}|^2}{2\mu^{(1)}}
    \right)
    =
    \frac{C_K}{\sqrt K}.
\]
Finally, since $v_k=s_k-\rho$,
\[
    \frac1K\sum_{k=1}^K s_k
    =
    \rho+\frac1K\sum_{k=1}^K v_k
    \le
    \rho+\frac{C_K}{\sqrt K}.
\]
This proves the claim.
\end{proof}
\subsection{Proof of Theorem~\ref{thm:joint_convergence}}
\label{app:proof_joint}

\begin{proof}
Let $w^*\in\arg\min_{w\in\mathcal W}\phi(w)$. Under
Assumption~\ref{ass:outer_convex}, $F(\cdot,\psi)$ is convex on
$\mathcal W$ for every $\psi$. Hence
\[
    \phi(w)=\max_{\psi\in\mathcal A}F(w,\psi)
\]
is convex as a pointwise maximum of convex functions.

Define
\[
    \Delta_K
    :=
    \frac{\max\{\rho,4R^2\}\mu^{(1)}}{\sqrt K}
\]
and
\[
    \mathcal R_K
    :=
    2L_f\sqrt{\rho+\gamma_M}
    +
    \frac{L_f C_K}{2\sqrt{(\rho+\gamma_M)K}}.
\]
We first show that the averaged inner gradient
$\hat g^{(t)}=K^{-1}\sum_{k=1}^K\hat g_k$ satisfies the approximate
subgradient inequality
\begin{equation}
    \phi(w^{(t)})-\phi^*
    \le
    \langle \hat g^{(t)},w^{(t)}-w^*\rangle
    +\Delta_K+\mathcal R_K .
    \label{eq:approx_sg_compact}
\end{equation}

For each $k$, convexity of $F(\cdot,\psi_k)$ gives
\[
    F(w^{(t)},\psi_k)-F(w^*,\psi_k)
    \le
    \langle \hat g_k,w^{(t)}-w^*\rangle .
\]
Averaging over $k$ yields
\begin{equation}
    \frac1K\sum_{k=1}^K
    \bigl[F(w^{(t)},\psi_k)-F(w^*,\psi_k)\bigr]
    \le
    \langle \hat g^{(t)},w^{(t)}-w^*\rangle .
    \label{eq:avg_convexity}
\end{equation}

Next, we compare $F(w^*,\psi_k)$ to $\phi^*$. Since
$\phi^*=F(w^*,\psi^*(w^*))$, Assumption~\textup{(A1)} and
Kantorovich--Rubinstein duality give
\[
    F(w^*,\psi_k)-\phi^*
    \le
    L_f W_1(
        \widehat P_{G_{\psi_k},x},
        \widehat P_{G_{\psi^*(w^*)},x}
    ).
\]
By the triangle inequality for $W_1$,
\[
\begin{aligned}
    W_1(
        \widehat P_{G_{\psi_k},x},
        \widehat P_{G_{\psi^*(w^*)},x}
    )
    &\le
    W_1(
        \widehat P_{G_{\psi_k},x},
        \widehat P_{G_{\hat\phi},x}
    )  \\
    &\quad+
    W_1(
        \widehat P_{G_{\hat\phi},x},
        \widehat P_{G_{\psi^*(w^*)},x}
    ).
\end{aligned}
\]
Using $W_1\le W_2$ and the empirical Sinkhorn--Wasserstein comparison
\eqref{eq:empirical_sw_gap}, we have
\[
    W_1(
        \widehat P_{G_{\psi_k},x},
        \widehat P_{G_{\hat\phi},x}
    )
    \le
    \sqrt{s_k+\gamma_M},
\]
while $\psi^*(w^*)\in\mathcal A$ implies
\[
    W_1(
        \widehat P_{G_{\hat\phi},x},
        \widehat P_{G_{\psi^*(w^*)},x}
    )
    \le
    \sqrt{\rho+\gamma_M}.
\]
Therefore,
\[
    F(w^*,\psi_k)-\phi^*
    \le
    L_f\left(\sqrt{s_k+\gamma_M}
    +
    \sqrt{\rho+\gamma_M}\right).
\]
Averaging the preceding pointwise bound over $k=1,\ldots,K$ gives
\begin{align}
    \frac1K\sum_{k=1}^K F(w^*,\psi_k)-\phi^*
    &=
    \frac1K\sum_{k=1}^K\bigl[F(w^*,\psi_k)-\phi^*\bigr] \notag\\
    &\le
    L_f\frac1K\sum_{k=1}^K\sqrt{s_k+\gamma_M}
    +L_f\sqrt{\rho+\gamma_M}.
    \label{eq:Rk_average}
\end{align}
Since $u\mapsto\sqrt{u}$ is concave, Jensen's inequality gives
\[
    \frac1K\sum_{k=1}^K\sqrt{s_k+\gamma_M}
    \le
    \sqrt{\frac1K\sum_{k=1}^K(s_k+\gamma_M)}
    =
    \sqrt{\frac1K\sum_{k=1}^K s_k+\gamma_M}.
\]
Substituting this into~\eqref{eq:Rk_average} yields
\begin{equation}
    \frac1K\sum_{k=1}^K F(w^*,\psi_k)-\phi^*
    \le
    L_f\sqrt{\frac1K\sum_{k=1}^K s_k+\gamma_M}
    +L_f\sqrt{\rho+\gamma_M}.
    \label{eq:Rk_jensen}
\end{equation}
By Lemma~\ref{lem:slack},
\[
    \frac1K\sum_{k=1}^K s_k
    \le
    \rho+\frac{C_K}{\sqrt K}.
\]
Substituting this bound into~\eqref{eq:Rk_jensen} gives
\begin{equation}
    \frac1K\sum_{k=1}^K F(w^*,\psi_k)-\phi^*
    \le
    L_f\sqrt{\rho+\gamma_M+\frac{C_K}{\sqrt K}}
    +L_f\sqrt{\rho+\gamma_M}.
    \label{eq:Rk_slack_substitution}
\end{equation}
Finally, using
\[
    \sqrt{a+b}\le \sqrt a+\frac{b}{2\sqrt a}
    \qquad(a>0,\ b\ge0),
\]
with $a=\rho+\gamma_M$ and $b=C_K/\sqrt K$, we have
\[
    \sqrt{\rho+\gamma_M+\frac{C_K}{\sqrt K}}
    \le
    \sqrt{\rho+\gamma_M}
    +
    \frac{C_K}{2\sqrt{(\rho+\gamma_M)K}}.
\]
Substituting this into~\eqref{eq:Rk_slack_substitution} gives
\begin{equation}
    \frac1K\sum_{k=1}^K F(w^*,\psi_k)-\phi^*
    \le
    2L_f\sqrt{\rho+\gamma_M}
    +
    \frac{L_f C_K}{2\sqrt{(\rho+\gamma_M)K}}
    =
    \mathcal R_K .
    \label{eq:Rk_bound}
\end{equation}

By Theorem~\ref{thm:inner_convergence},
\begin{equation}
    \phi(w^{(t)})
    =
    F(w^{(t)},\psi^*(w^{(t)}))
    \le
    \frac1K\sum_{k=1}^K F(w^{(t)},\psi_k)
    +\Delta_K .
    \label{eq:inner_value_joint}
\end{equation}
Combining \eqref{eq:inner_value_joint}, \eqref{eq:Rk_bound}, and
\eqref{eq:avg_convexity} gives
\[
\begin{aligned}
    \phi(w^{(t)})-\phi^*
    &\le
    \frac1K\sum_{k=1}^K F(w^{(t)},\psi_k)
    -\phi^*
    +\Delta_K \\
    &=
    \frac1K\sum_{k=1}^K
    \bigl[F(w^{(t)},\psi_k)-F(w^*,\psi_k)\bigr]
    +
    \left(
        \frac1K\sum_{k=1}^K F(w^*,\psi_k)-\phi^*
    \right)
    +\Delta_K \\
    &\le
    \frac1K\sum_{k=1}^K
    \bigl[F(w^{(t)},\psi_k)-F(w^*,\psi_k)\bigr]
    +\Delta_K+\mathcal R_K \\
    &\le
    \langle \hat g^{(t)},w^{(t)}-w^*\rangle
    +\Delta_K+\mathcal R_K .
\end{aligned}
\]
which proves \eqref{eq:approx_sg_compact}.

Now use the projected outer update
\[
    w^{(t+1)}
    =
    \Pi_{\mathcal W}
    \bigl(w^{(t)}-\eta_w\hat g^{(t)}\bigr).
\]
By nonexpansiveness of projection and, using
Assumption~\ref{ass:outer_convex},
$\|\hat g^{(t)}\|\le K^{-1}\sum_k\|\hat g_k\|\le G_w$,
\[
\begin{aligned}
    \|w^{(t+1)}-w^*\|^2
    &\le
    \|w^{(t)}-\eta_w\hat g^{(t)}-w^*\|^2 \\
    &=
    \|w^{(t)}-w^*\|^2
    -2\eta_w
    \langle \hat g^{(t)},w^{(t)}-w^*\rangle
    +\eta_w^2\|\hat g^{(t)}\|^2 \\
    &\le
    \|w^{(t)}-w^*\|^2
    -2\eta_w
    \langle \hat g^{(t)},w^{(t)}-w^*\rangle
    +\eta_w^2G_w^2 .
\end{aligned}
\]
Rearranging and using \eqref{eq:approx_sg_compact},
\begin{equation}
    \phi(w^{(t)})-\phi^*
    \le
    \frac{\|w^{(t)}-w^*\|^2-\|w^{(t+1)}-w^*\|^2}{2\eta_w}
    +
    \frac{\eta_wG_w^2}{2}
    +
    \Delta_K+\mathcal R_K .
    \label{eq:one_step_compact}
\end{equation}

Summing \eqref{eq:one_step_compact} over $t=1,\ldots,T$ and using
$\|w^{(1)}-w^*\|\le D_{\mathcal W}$ gives
\[
    \frac1T\sum_{t=1}^T[\phi(w^{(t)})-\phi^*]
    \le
    \frac{D_{\mathcal W}^2}{2\eta_w T}
    +
    \frac{\eta_wG_w^2}{2}
    +
    \Delta_K+\mathcal R_K .
\]
By convexity of $\phi$,
\[
    \phi(\bar w^{(T)})\le \frac1T\sum_{t=1}^T\phi(w^{(t)}).
\]
Therefore,
\[
    \phi(\bar w^{(T)})-\phi^*
    \le
    \frac{D_{\mathcal W}^2}{2\eta_w T}
    +
    \frac{\eta_wG_w^2}{2}
    +
    \Delta_K+\mathcal R_K .
\]
Finally, substituting
\[
    \eta_w=\frac{D_{\mathcal W}}{G_w\sqrt T}
\]
gives
\[
    \frac{D_{\mathcal W}^2}{2\eta_w T}
    +
    \frac{\eta_wG_w^2}{2}
    =
    \frac{D_{\mathcal W}G_w}{\sqrt T}.
\]
Substituting the definitions of $\Delta_K$ and $\mathcal R_K$ yields
\[
\begin{aligned}
    \phi(\bar w^{(T)})-\phi^*
    \le\;
    &
    \frac{D_{\mathcal W}G_w}{\sqrt T}
    +
    \frac{\max\{\rho,4R^2\}\mu^{(1)}}{\sqrt K} \\
    &+
    \frac{L_f C_K}{2\sqrt{(\rho+\gamma_M)K}}
    +
    2L_f\sqrt{\rho+\gamma_M},
\end{aligned}
\]
which is the desired bound.
\end{proof}


\section{Experimental Details}
\label{app:experiments}

We evaluate GDRO on two downstream decision problems.  The first is
contextual newsvendor, where the decision is convex and the optimal
order for a fixed scenario distribution has a weighted-quantile form.
We study this task on both a synthetic benchmark, used to validate the
rare-tail mechanism under a known data-generating process, and the real
M5 retail demand dataset.  The second downstream problem is robot
navigation among pedestrians, where the decision is nonconvex and the
nominal model is the official implicit SocialGAN trajectory generator.
Across all experiments, the nominal generator defines $P_0^x$, the
adversarial generator defines $Q_\psi^x$, and each method chooses a
decision $w\in\mathcal W$ before the realized outcome $Y$ is revealed.

\subsection{Baselines and Solvers}
\label{app:baselines}

All methods use the same fitted nominal generator within each
benchmark.  They differ only in how they form the scenario distribution
used by the downstream decision solver.

\paragraph{Nominal.}
The nominal baseline optimizes the decision using Monte Carlo samples
from $P_0^x$.  This is the standard sample-average decision induced by
the fitted generator.

\paragraph{KL.}
The KL baseline follows the empirical $f$-divergence DRO formulation of
\citet{namkoong2016stochastic}.  At each context, it keeps the nominal
sample cloud fixed and adversarially reweights the samples within a KL
ball.  For a fixed decision, the KL-constrained worst-case weights have
the exponential-tilt form \citep{hu2013kullback}; we choose the tilt
temperature to match the prescribed radius.  The downstream decision is
then updated under the weighted samples: by weighted critical quantiles
for newsvendor and by weighted first-order optimization for SocialGAN.
Thus this baseline can emphasize high-loss nominal samples but cannot
create outcomes outside the nominal sample cloud.

\paragraph{W2.}
The W2 baseline adapts the adversarial sample-perturbation idea of
\citet{sinha2018certifying} to the downstream decision stage as a
finite-sample Wasserstein-2 robustness baseline.  Given nominal
scenarios from $P_0^x$, we introduce adversarial support points and
move them directly in $\mathcal Y$ to increase the downstream loss,
while controlling their average squared displacement from the nominal
cloud, i.e., the sample-wise empirical $W_2^2$ transport cost.  Unlike
the penalized robust-training objective in \citet{sinha2018certifying},
we use a primal--dual update for an explicit transport-radius
constraint, so that the baseline is comparable to the radius-constrained
GDRO formulation.  The downstream decision is then updated against the
perturbed scenario cloud.  This gives a uniform finite-sample W2
baseline for both the multi-product newsvendor task and the nonconvex
SocialGAN planning task.  Unlike GDRO, however, the adversarial
scenarios are free support points in $\mathcal Y$ rather than samples
produced by the chosen generator family.

\paragraph{GAS-DRO.}
For the newsvendor experiments, we include a conditional version of
GAS-DRO \citep{wen2026gasdro}.  The adversarial model uses the same
conditional decoder architecture as the nominal generator and is
optimized at the active context.  Its ambiguity set is controlled by
the GAS-DRO reconstruction certificate, computed over conditional
training pairs.  We denote by $J_0$ the value of this reconstruction
certificate for the fitted nominal generator on the training set, and
set GAS-DRO radii as multiples of $J_0$ in the experiments below.  We
include this baseline only when an explicit decoder is available; we do
not apply it to SocialGAN, whose released nominal model is an implicit
trajectory sampler.

\paragraph{GDRO.}
We use the empirical context-local Sinkhorn formulation from
Section~\ref{sec:method}.  The nominal generator is frozen, the
adversarial generator is initialized from the nominal architecture, and
the inner player updates adversarial generator parameters subject to
the Sinkhorn certificate at the active context.  The decision update
uses the average outer subgradient over the active inner adversaries. All Sinkhorn divergences are computed with GeomLoss\citep{feydy2019sinkhorn} using $p=2$ and blur $0.05$ across experiments. Task-specific radii, sample sizes, and optimization budgets are listed
in the configuration tables below.

\paragraph{Radius and optimization budget.}
For the synthetic newsvendor experiment, we use fixed radii for KL,
W2, and GDRO because the simulator is used as a controlled
mechanism study.  For the real-data experiments, the robustness level
is adapted to the active context.  In M5, we compute a historical
average-cost risk score from the observed demand history and use its
training-set $10\%$ and $90\%$ quantiles as low- and high-risk
anchors.  In SocialGAN, we use the nominal collision probability of the
nominal robot plan as the risk score.  The KL, W2, and GDRO
radii are increased for high-risk contexts and decreased for low-risk
contexts by a monotone validation-calibrated rule.  The inner
optimization budget is adapted in the same direction: high-risk
contexts receive more adversarial updates, while low-risk contexts use
a smaller budget.  These adaptations use only the observed context and
nominal generated scenarios, never realized test outcomes.

\subsection{Synthetic Contextual Newsvendor}
\label{app:synthetic}

\paragraph{Motivation.}
The synthetic benchmark is designed to test a specific failure mode of
nominal generative decision-making: a conditional generator may fit the
bulk of the demand distribution while underrepresenting rare,
high-cost demand tails.  The data-generating process therefore contains
frequent base contexts and rare contexts in which a tail component may
activate.  This lets us evaluate decisions against the true conditional
law while still requiring every method to make decisions from the same
learned nominal generator.

\paragraph{Data-generating process.}
We generate $p=20$ product demands from contexts
$x\in\mathbb R^5$, with
\[
    x\sim \mathcal N(0,I).
\]
A context is rare if
\[
    \mathrm{rare}(x)
    =
    \mathbf 1\{a^\top x>\tau\},
\]
where $a$ is a fixed unit vector and $\tau$ is the $0.9$ training
quantile of $a^\top x$.

Demand is generated through a latent factor $h\in\mathbb R^2$.  In the
base component,
\[
    h\mid x \sim \mathcal N(Ax,\Sigma_h).
\]
In rare contexts, a tail component is activated with probability
$\pi_{\mathrm{tail}}=0.10$:
\[
    h\mid x,\mathrm{tail}
    \sim
    \mathcal N(Ax+\Delta_h,\Sigma_{\mathrm{tail}}),
    \qquad
    \Delta_h=(2.0,0).
\]
Frequent contexts always use the base component.  For product $j$, the
latent score is
\[
    s_j(x,h)
    =
    b_j+r_j^\top x+\beta_j h_1+\gamma_j h_2
    +a_j^{\sin}\sin(v_j^\top h)+\epsilon_j,
    \qquad
    \epsilon_j\sim\mathcal N(0,\sigma_\epsilon^2),
\]
and demand is
\[
    Y_j
    =
    \mu_0+\mu_1\,\mathrm{softplus}(s_j(x,h)).
\]
We use $\sigma_\epsilon=0.15$, $\mu_0=10$, and $\mu_1=5$.  Product
parameters share group-level structure plus product-specific
perturbations, inducing correlated nonlinear demands across products.

\paragraph{Nominal generator.}
The nominal model is a conditional VAE trained on samples from the
data-generating process.  The context is provided to the encoder and
decoder.  Demand is standardized for VAE training, but generated
samples are transformed back to demand units before the downstream
decision and evaluation. The full synthetic newsvendor configuration
is summarized in Table~\ref{tab:synthetic-config}.

\paragraph{Downstream decision.}
The decision is the order vector $w\in\mathbb R_+^{20}$.  The loss is
the multi-product newsvendor cost
\[
    f(w,Y)
    =
    \sum_{j=1}^{20}
    c_{u,j}(Y_j-w_j)_+
    +
    c_{o,j}(w_j-Y_j)_+ .
\]
For any weighted empirical scenario cloud, the optimal decision is the
per-product weighted critical quantile with critical ratio
$c_{u,j}/(c_{u,j}+c_{o,j})$.  The oracle is computed using simulator
samples from the true conditional law at the same context and is used
only for evaluation.

\begin{table}[t]
\centering
\caption{Synthetic contextual newsvendor configuration.}
\label{tab:synthetic-config}
\begin{tabular}{ll}
\toprule
Item & Value \\
\midrule
Products & $20$ \\
Context dimension & $5$ \\
DGP latent dimension & $2$ \\
Rare-context threshold & $0.9$ quantile of $a^\top x$ \\
Tail probability in rare contexts & $0.10$ \\
Tail shift & $(2.0,0)$ \\
Observation noise & $0.15$ \\
Demand scale $(\mu_0,\mu_1)$ & $(10,5)$ \\
Train / validation / test contexts & $2000 / 500 / 200$ \\
Underage costs $c_{u,j}$ & Uniform on $[15,25]$ \\
Overage costs $c_{o,j}$ & Uniform on $[1,3]$ \\
Nominal generator & Conditional VAE \\
VAE latent dimension & $4$ \\
VAE hidden width & $128$ \\
VAE epochs & $300$ \\
VAE learning rate & $10^{-3}$ \\
VAE KL weight & $0.02$ \\
Nominal/KL samples per context & $256$ \\
GDRO/W2/GAS-DRO samples per context & $128$ \\
KL radius & $0.1$ \\
GDRO radius & $2.5$ \\
W2 radius & $2.5$ \\
GAS-DRO radius & $0.2J_0$ \\
\bottomrule
\end{tabular}
\end{table}


\subsection{M5 Retail Newsvendor}
\label{app:m5}

\paragraph{Dataset and product selection.}
The M5 experiment uses real retail demand from the M5 forecasting
benchmark \citep{makridakis2022m5}.  Our goal in this experiment is not to compete on the full
M5 forecasting task, but to obtain a stable real-data nominal generator
for testing the downstream newsvendor decision mechanism.  We therefore
construct a contextual multi-product newsvendor task from historical
sales, and select $p=20$ products from the
chosen M5 subset whose nonzero demand ratio is at least $0.95$.  This
keeps the task multivariate while avoiding products dominated by zeros,
which would make conditional scenario generation unstable and obscure
the effect of the robust decision layer.

\paragraph{Downstream decision.}
The downstream decision is the same multi-product newsvendor decision
as in Appendix~\ref{app:synthetic}.  Let
$Y_t=(Y_{t,1},\ldots,Y_{t,p})^\top$ denote the realized demand vector
for $p=20$ products on day $t$, and let
$w_t=(w_{t,1},\ldots,w_{t,p})^\top\in\mathbb R_+^p$ denote the order
decision.  The held-out realized loss is
\[
    f(w_t,Y_t)
    =
    \sum_{j=1}^{20}
    c_{u,j}(Y_{t,j}-w_{t,j})_+
    +
    c_{o,j}(w_{t,j}-Y_{t,j})_+ .
\]
For any weighted empirical scenario cloud, the decision is computed by
the per-product weighted critical quantile.

\paragraph{Nominal generator.}
The nominal model is a conditional VAE--LSTM demand generator.  For
each day $t$, the input context $X_t$ is a historical demand window.  A single-layer LSTM encodes this
history into a temporal representation $\widehat H_t$.  During
training, the VAE encoder conditions on $(\widehat H_t,Y_t)$ and
learns an approximate posterior over the latent variable.  At
generation time, the future demand $Y_t$ is not observed; latent
variables are sampled from the learned conditional prior given
$\widehat H_t$, and the decoder produces next-day demand scenarios.
The nominal CVAE--LSTM architecture is illustrated in
Figure~\ref{fig:m5_nominal_cvae}.

\begin{figure}[t]
    \centering
    \includegraphics[width=\textwidth]{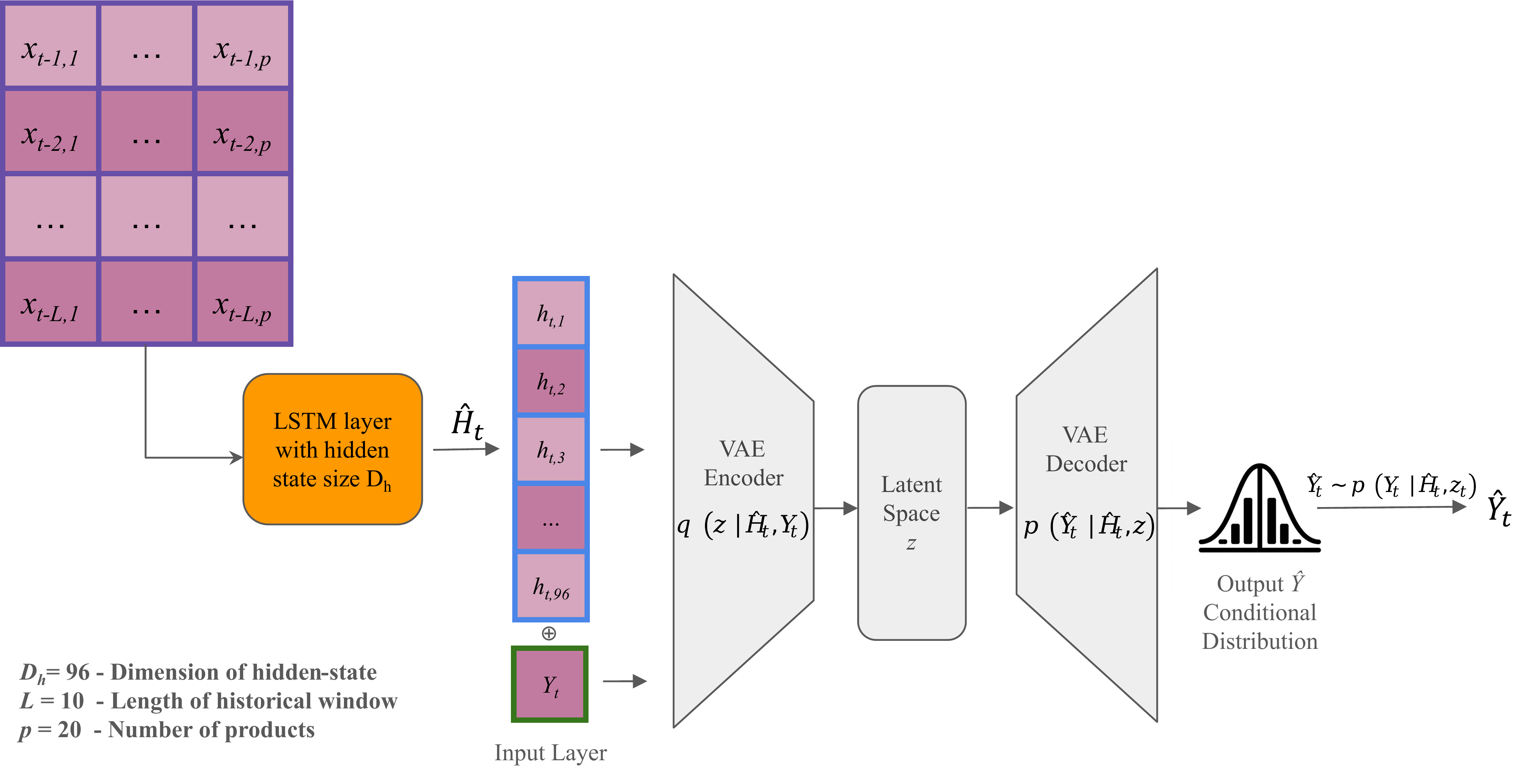}
    \caption{
        Nominal CVAE--LSTM generator for M5 demand scenarios. The LSTM
        encodes the historical demand window into
        $\widehat H_t$.  During training, the encoder uses
        $(\widehat H_t,Y_t)$; during generation, latent variables are
        sampled from the conditional prior given $\widehat H_t$ and
        decoded into next-day demand scenarios.
    }
    \label{fig:m5_nominal_cvae}
\end{figure}

\paragraph{Adaptive robustness.}
Unlike the synthetic benchmark, M5 does not provide a known simulator
mechanism for identifying rare-tail contexts.  We therefore adapt the
robustness level using a risk score computed from the same newsvendor
objective as the downstream task.  For each context $x$, corresponding
to a historical demand window, we compute $r(x)$ as the average
realized newsvendor cost over that window under the same underage and
overage costs used for evaluation.  Let $r_{0.10}$ and $r_{0.90}$ be
the training-set $10\%$ and $90\%$ quantiles of this score, and define
\[
    \alpha(x)
    =
    \left[
    \frac{r(x)-r_{0.10}}{r_{0.90}-r_{0.10}}
    \right]_{[0,1]} .
\]
The KL, W2, and GDRO radii are interpolated between their low- and
high-risk values using $\alpha(x)$.  For methods with adversarial inner
optimization, the inner-loop budget is increased using the same risk
score.  All thresholds, radius endpoints, and optimization budgets are
chosen on validation contexts.  The full M5 configuration is summarized
in Table~\ref{tab:m5-config}.

\begin{table}[t]
\centering
\caption{M5 retail newsvendor configuration.}
\label{tab:m5-config}
\begin{tabularx}{0.9\linewidth}{@{}Xl@{}}
\toprule
Item & Value \\
\midrule
Products & $20$ \\
Historical window length & $10$ days \\
Input-context shape & $10 \times 20$ \\
Forecast horizon & $1$ day \\
Train / validation / test contexts & $1500/100/300$ \\
Nominal generator & Conditional VAE--LSTM \\
LSTM layers & $1$ \\
LSTM hidden-state dimension & $96$ \\
VAE latent dimension & $12$ \\
VAE hidden dimension & $192$ \\
Maximum VAE epochs & $600$ \\
Early-stopping patience & $120$ \\
VAE batch size & $64$ \\
VAE learning rate & $5\times10^{-4}$ \\
VAE weight decay & $10^{-5}$ \\
VAE KL weight & $10^{-3}$ \\
Prior-mean loss weight & $1$ \\
Nominal/KL samples per context & $256$ \\
GDRO/W2/GAS--DRO samples per context & $128$ \\
Underage costs $c_{u,j}$ & $20$ \\
Overage costs $c_{o,j}$ & $2$ \\
KL radius & $[0.14, 1.99]$ \\
W2 radius & $[1.0, 58.0]$ \\
GDRO radius & $[38.0, 75.0]$ \\
GAS--DRO radius & $0.80J_0$ \\
\bottomrule
\end{tabularx}
\end{table}


\subsection{SocialGAN Robot Navigation}
\label{app:socialgan}

\paragraph{Nominal model and data.}
We use the official pretrained SocialGAN pedestrian trajectory model
and the ETH/UCY benchmark protocol from \citet{gupta2018social}.  The
model observes $8$ frames of pedestrian history and generates $12$
future frames.  We use the released SocialGAN predictor and its test
split, without retraining the trajectory generator.

\paragraph{Downstream decision.}
The robot navigation task is added by us on top of the SocialGAN
prediction setting.  At a test context $x$, the robot observes the same
past pedestrian trajectories as SocialGAN.  It then chooses a
$12$-step velocity plan
\[
    w=(u_1,\ldots,u_{12}), \qquad u_t\in\mathbb R^2,
\]
before the true pedestrian futures are revealed.  The robot follows
single-integrator dynamics from a sampled start point to a sampled goal,
with each velocity projected onto a maximum-speed ball.  Optimization
uses generated pedestrian futures from the nominal or robust model;
evaluation uses the true held-out pedestrian futures.

\paragraph{Planning objective and evaluation.}
For a pedestrian-future scenario
$Y=\{y_{j,t}\}_{j=1,t=1}^{N,T}$, the robot decision is a
$T$-step velocity plan
\[
    w=(u_1,\ldots,u_T), \qquad u_t\in\mathbb R^2,\qquad T=12 .
\]
Starting from $s$, the induced robot path is
\[
    r_t(w)=s+\sum_{\ell=1}^t u_\ell,
    \qquad
    \|u_t\|_2\le v_{\max}.
\]
The speed constraint is enforced by projecting each velocity vector
onto the Euclidean ball of radius $v_{\max}$.  Let $g$ be the sampled
goal and let $\bar r_t$ be the straight-line reference path from $s$ to
$g$.  The smooth loss optimized by all methods has the form
\[
\begin{aligned}
    f(w,Y)
    =
    &\underbrace{
    \lambda_p {1\over T}\sum_{t=1}^T \|r_t(w)-g\|_2^2
    + \lambda_T \|r_T(w)-g\|_2^2
    }_{\text{progress and terminal arrival}} \\
    &+
    \underbrace{
    \lambda_{\rm tr}{1\over T}\sum_{t=1}^T
    \|r_t(w)-\bar r_t\|_2^2
    }_{\text{weak reference-path tracking}} \\
    &+
    \underbrace{
    \lambda_u {1\over T}\sum_{t=1}^T \|u_t\|_2^2
    + \lambda_s {1\over T-1}\sum_{t=2}^T
    \|u_t-u_{t-1}\|_2^2
    }_{\text{control effort and smoothness}} \\
    &+
    \underbrace{
    \lambda_c {1\over T}\sum_{t=1}^T
    \operatorname{softplus}
    \left(
    {r_{\rm safe}-\widetilde d_t(w,Y)\over \tau_b}
    \right)^2
    }_{\text{differentiable collision barrier}} .
\end{aligned}
\]
Here
\[
    \widetilde d_t(w,Y)
    =
    -\tau_d
    \log\sum_{j=1}^N
    \exp\left(
    -{\|r_t(w)-y_{j,t}\|_2\over \tau_d}
    \right)
\]
is a soft-min approximation of the nearest-pedestrian distance at time
$t$.  The first line encourages the robot to make progress and arrive
at the goal by the end of the horizon.  The second line weakly
discourages unnecessary detours.  The third line penalizes large and
rapidly changing velocities.  The last line is a smooth safety barrier:
it becomes large when the robot comes within the safety radius
$r_{\rm safe}$ of a sampled pedestrian future.  Nominal, W2, and GDRO
optimize the average of $f(w,Y)$ over their corresponding scenario
clouds, while KL optimizes a weighted average under adversarial sample
weights.

For final evaluation, we use the true held-out pedestrian future and a
hard, non-differentiable score that prioritizes safety and arrival:
\[
\begin{aligned}
    \mathrm{Eval}(w,Y^{\rm true})
    =
    &\;
    C_{\rm col}\mathbf 1\{d_{\min}<r_{\rm col}\}
    + C_{\rm near}\mathbf 1\{d_{\min}<r_{\rm near}\} \\
    &+
    C_{\rm fail}\mathbf 1\{t_{\rm arr}=-1\}
    + c_{\rm time}\widetilde t_{\rm arr}
    + c_{\rm len}L_{\rm path}
    + c_{\rm dev}D_{\rm ref}.
\end{aligned}
\]
Here $d_{\min}$ is the minimum robot--pedestrian distance before
arrival, $t_{\rm arr}$ is the first frame at which
$\|r_t(w)-g\|_2\le r_{\rm arr}$, $\widetilde t_{\rm arr}=t_{\rm arr}$
if the robot arrives and $T+1$ otherwise, $L_{\rm path}$ is path length
until arrival, and $D_{\rm ref}$ is average deviation from the
straight-line reference.  Once the robot reaches the goal, later
pedestrian motion is not counted as a collision.  This hard score is
used only for reporting; optimization uses the smooth loss above.

\paragraph{Fixed task set.}
We uniformly select $50$ contexts from the SocialGAN test split.  For
each selected context, we generate $4$ start--goal queries, giving
$200$ fixed robot tasks.  Queries are selected using only nominal
SocialGAN samples: the straight-line robot path must have nominal
collision probability in $[0.60,0.90]$ under a screening radius $0.30$,
and the start and goal locations must be initially clear of pedestrians.
The evaluation collision radius is $0.50$.  Thus the benchmark focuses
on nontrivial planning instances without using true future trajectories
during task construction.

\paragraph{Adaptive radius and budget.}
For each query, we first solve the nominal robot plan and estimate its
nominal collision probability $\widehat p$ using generated futures and
the evaluation collision radius.  The GDRO and W2 radii are chosen by
the clipped linear rule
\[
    \rho(\widehat p)
    =
    \rho_{\min}
    +
    (\rho_{\max}-\rho_{\min})
    \left[
    \frac{\widehat p-p_{\min}}{p_{\max}-p_{\min}}
    \right]_{[0,1]},
\]
with $\rho_{\min}=0.5$, $\rho_{\max}=3.0$,
$p_{\min}=0.05$, and $p_{\max}=0.50$.  The KL radius is adapted by the
same rule from $0.1$ to $0.7$.  The GDRO optimization budget is also
adapted using $\widehat p$: low-risk queries use the smaller budget in
Table~\ref{tab:socialgan-config}, while high-risk queries use the
larger budget.  All adaptive quantities use only the observed context
and nominal generated futures, never the realized test future.

\begin{table}[t]
\centering
\caption{SocialGAN robot navigation configuration.}
\label{tab:socialgan-config}
\begin{tabular}{ll}
\toprule
Item & Value \\
\midrule
Nominal model & Official SocialGAN \citep{gupta2018social} \\
Dataset protocol & ETH/UCY SocialGAN test split \\
Observed / predicted frames & $8 / 12$ \\
Selected test contexts & $50$ uniformly selected contexts \\
Queries per context & $4$ \\
Total robot tasks & $200$ \\
Robot decision horizon $T$ & $12$ \\
Robot decision $w$ & $12$ two-dimensional velocity controls \\
Robot dynamics & $r_t=s+\sum_{\ell=1}^t u_\ell$ \\
Robot speed limit $v_{\max}$ & $2\times$ reference pedestrian speed \\
Velocity constraint & Projection to $\|u_t\|_2\le v_{\max}$ \\
Straight-line screening probability & $[0.60,0.90]$ \\
Screening collision radius & $0.30$ \\
Evaluation collision radius $r_{\rm col}$ & $0.50$ \\
Near-miss radius $r_{\rm near}$ & $0.80$ \\
Arrival radius $r_{\rm arr}$ & $0.25$ \\
Nominal SocialGAN samples per context & $64$ \\
Nominal plan steps & $600$ \\
GDRO/W2 robust samples per context & $8$ fixed noise samples \\
GDRO/W2 radius range $\rho$ & $[0.5, 3.0]$ \\
KL radius range & $[0.1,0.7]$ \\
Risk anchors $(p_{\min},p_{\max})$ & $(0.05,0.50)$ \\
Low-risk GDRO budget & $150$ outer, $3$ dual cycles, $5$ adversary steps \\
High-risk GDRO budget & $500$ outer, $5$ dual cycles, $8$ adversary steps \\
Controls updates per outer step & $5$ \\
Sinkhorn blur $\varepsilon$ & $0.05$ \\
Loss scale in adversarial update & $25.0$ \\
Progress weight $\lambda_p$ & $1.2$ \\
Terminal arrival weight $\lambda_T$ & $35.0$ \\
Reference tracking weight $\lambda_{\rm tr}$ & $0.05$ \\
Control effort weight $\lambda_u$ & $0.15$ \\
Control smoothness weight $\lambda_s$ & $0.75$ \\
Collision barrier weight $\lambda_c$ & $80.0$ \\
Safety radius in smooth loss $r_{\rm safe}$ & $0.55$ \\
Barrier temperature $\tau_b$ & $0.10$ \\
Soft-min temperature $\tau_d$ & $0.10$ \\
Decision learning rate & $0.035$ \\
Adversarial generator learning rate & $10^{-6}$ \\
Initial dual value / GDRO dual learning rate & $5.0 / 10.0$ \\
Gradient clipping & $1.0$ \\
W2 support-point learning rate & $0.015$ \\
W2 dual learning rate & $2.0$ \\
W2 outer / support / control steps & $180 / 15 / 5$ \\
Evaluation collision penalty $C_{\rm col}$ & $100.0$ \\
Evaluation near-miss penalty $C_{\rm near}$ & $20.0$ \\
Evaluation failure penalty $C_{\rm fail}$ & $50.0$ \\
Evaluation time / path / deviation weights $(c_{\rm time},c_{\rm len},c_{\rm dev})$ & $1.0 / 1.0 / 0.2$ \\
\bottomrule
\end{tabular}
\end{table}

\end{document}